\documentclass[11pt]{article}
\usepackage{hyphenat}
\usepackage[in]{fullpage}
\usepackage{multirow}
\usepackage{amsthm}
\newcommand{\algc}[1]{\hfill\(\triangleright\) {\footnotesize #1}}
\usepackage{times}
\usepackage{helvet}
\usepackage{courier}
\usepackage{graphicx}
\usepackage[cmex10]{amsmath}
\usepackage[mathscr]{euscript}
\usepackage{bm}
\usepackage{amsfonts}
\usepackage{tikz}
\usepackage[mathscr]{euscript}
\usepackage{color}
\usepackage[dvipsnames]{xcolor}
\usepackage{hyperref}
\usepackage{cleveref}
\crefname{algorithm}{Algorithm}{Algorithms}
\Crefname{algorithm}{Algorithm}{Algorithms}
\crefname{assumption}{Assumption}{Assumptions}
\Crefname{assumption}{Assumption}{Assumptions}
\usepackage{nicefrac}       
\usepackage{microtype}      
\usepackage{xcolor}
\usepackage[most]{tcolorbox}
\usepackage{wrapfig}

\definecolor{bestcolor}{named}{Green}

\usepackage{capt-of}

\newenvironment{squishenumerate}
{\begin{list}{\arabic{enumi}.}{%
    \usecounter{enumi}%
    \setlength{\itemsep}{2pt}%
    \setlength{\parsep}{1pt}%
    \setlength{\topsep}{3pt}%
    \setlength{\parskip}{0pt}%
    \setlength{\labelwidth}{.5in}%
    \setlength{\labelsep}{0.05in}%
    \setlength{\leftmargin}{.2in}}}
{\end{list}}

\usepackage{amsmath,mleftright}
\hypersetup{
    pdffitwindow=true,
    pdfstartview={FitH},
    pdfnewwindow=true,
    colorlinks,
    linktocpage=true,urlcolor=Green,
    linkcolor=Red,
    citecolor=Blue
}

\usepackage{bm}
\usepackage[american]{babel}
\usepackage{xr}

\usepackage[utf8]{inputenc} 
\usepackage[T1]{fontenc}    
\usepackage{url}            
\usepackage{booktabs}       
\usepackage{amsfonts}       
\usepackage{nicefrac}       
\usepackage{microtype}      
\usepackage{float}

\usepackage{parskip} 

\usepackage{colonequals}
\usepackage{textcomp}
\usepackage{amsopn,float,bbm,bm,enumerate,color,multirow,gensymb}
\usepackage{amsfonts,amsmath,amssymb,amsthm} 
\usepackage{array}
\usepackage{url}            
\usepackage{booktabs}       
\usepackage{nicefrac}       
\usepackage{microtype}      
\usepackage{bm}

\usepackage{natbib}
\usepackage{graphicx}
\usepackage{subfigure}

\usepackage{xspace}
\usepackage{bigints}
\usepackage[utf8]{inputenc} 
\usepackage[T1]{fontenc}    

\usepackage{epsfig,subfigure,graphicx}
\usepackage{comment}

\usepackage{afterpage}
\usepackage{thmtools,thm-restate}

\usepackage{enumitem}

\usepackage{amsthm}

\newtheorem{nad}{Notation and Definitions}[section]

\newtheorem{theorem}{Theorem}[section]

\usepackage{algorithm} 
\usepackage{algpseudocode}

\usepackage[]{color-edits}
\addauthor[]{rd}{violet}
\addauthor[]{ug}{teal}
\addauthor[]{ks}{Green}

\usepackage{url}
\usepackage{cleveref}

\title{Inference Time Policy Optimization for\\ Offline RL with Differentiable World Models}

\author{\normalsize 
\makebox[\textwidth][c]{%
\begin{tabular}{c@{\hspace{2cm}}c}
\textbf{Rohan Deb} & \textbf{Steven J. Wright} \\
Siebel School of Computing and Data Science & Department of Computer Sciences\\
University of Illinois Urbana-Champaign & University of Wisconsin Madison \\
{\texttt{rd22@illinois.edu}} &
{\texttt{sjwright2@wisc.edu}}
\end{tabular}
}
\\[1.2em]
\makebox[\textwidth][c]{%
\begin{tabular}{c}
\textbf{Arindam Banerjee} \\
Siebel School of Computing and Data Science \\
University of Illinois Urbana-Champaign \\
{\texttt{arindamb@illinois.edu}}
\end{tabular}
}
}

\date{}
\usepackage[most]{tcolorbox}

\newcommand{\beq}{\begin{equation}}
\newcommand{\eeq}{\end{equation}}

\DeclareMathOperator{\argmax}{argmax}

\newcommand {\commentout}[1] {}

\def\ints{{{\rm Z} \kern -.35em {\rm Z} }}  
\def\smallints{{{\rm Z} \kern -.3em {\rm Z} }}  
\def\pints{{{\rm I} \kern -.15em {\rm N} }}      
\newcommand{\reals}{\mathbb R}

\def\cplx{{{\rm I} \kern -.45em {\rm C} }}       
\def\l2{\rm {\mathcal L}^{2}(\reals)}            

\newcommand{\be}{\begin{eqnarray}}
\newcommand{\ee}{\end{eqnarray}}
\newcommand{\bea}{\begin{eqnarray}}
\newcommand{\eea}{\end{eqnarray}}
\newcommand{\beaa}{\begin{eqnarray*}}
\newcommand{\eeaa}{\end{eqnarray*}}
\newcommand{\bnad}{\begin{nad}}
\newcommand{\enad}{\end{nad}}

\newcommand{\IGNORE}[1]{}

﻿

{\begin{array}{l@{\hspace{8mm}}l@{\hspace{8mm}}l}}%
{\end{array}}
\newlength{\lplb}
\DeclareMathSymbol{\mhyphen}{\mathord}{AMSa}{"39}

\usepackage{thmtools,thm-restate}

\theoremstyle{definition}

\definecolor{darkgreen}{rgb}{0.05, 0.5, 0.25}

\newcommand{\DiffMPC}{\color{Green}\ensuremath{\mathtt{DiffMPC}}}
\newcommand{\FlowMPC}{\color{Green}\ensuremath{\mathtt{FlowMPC}}}

\newenvironment{greenboxheadless}[1][]{
    \refstepcounter{innerdef}%
    \begin{tcolorbox}[
        enhanced=false,
        breakable=false,
        title=,
        notitle,
        left=5pt,
        right=5pt,
        top=5pt,
        bottom=5pt,
        colback=darkgreen!1,
        colframe=darkgreen!150,
        boxrule=1pt,
        arc=2pt,
        outer arc=2pt
    ]
}{
    \end{tcolorbox}
}
\begin{document}

\maketitle

\begin{abstract}

Offline Reinforcement Learning (RL) learns optimal policies from fixed datasets, training a policy once and deploying it at inference time without further refinement. Inspired by model predictive control (MPC), we introduce an inference time adaptation framework  that utilizes a pretrained policy along with a learned world model. 
While existing world model and diffusion-planning methods use learned dynamics to generate imagined trajectories during training, or to sample candidate plans at inference time, they do not use inference-time information to \emph{optimize} the policy parameters on the fly.
In contrast, our design is a {Differentiable World Model} (DWM) pipeline that enables end-to-end gradient computation through imagined rollouts for {\em inference time policy optimization} (ITPO). 
We evaluate our algorithm on D4RL continuous-control benchmarks (MuJoCo locomotion tasks and AntMaze), and show that exploiting inference-time information to optimize the policy parameters yields consistent gains over strong offline RL baselines. Inference-time adaptation, however, is expensive: rollout generation and backpropagation dominate per-step compute. We study this tradeoff explicitly, showing that a suitable tilted version of one-step MeanFlow sampler recovers much of the gains at a fraction of the cost.


\end{abstract}

\section{Introduction}
\label{sec:introduction}
\vspace{-0.5ex}
Offline reinforcement learning learns a policy in a Markov decision process from a fixed dataset of previously collected trajectories, without any additional interaction with the environment \citep{levine2020offline,prudencio2022offline,ernst2005tree,fujimoto2019offpolicy}. 
This setting arises naturally in domains where data is abundant but online interaction is constrained \citep{levine2020offline,fujimoto2019offpolicy,liu2020provably}. 
Examples include modern recommender and advertising systems, which operate on massive logs of user interactions \citep{xiao2023interactive,afsar2021rlrs,cai2017rtb,zhao2021dear}; healthcare and other safety-critical decision problems, which cannot permit exploratory actions \citep{gottesman2019guidelines,yu2021healthcare,fatemi2022semimarkov}; robotics and control, which may be limited by real-world cost, wear, and safety \citep{kober2013robotics,zhou2022realworld,venkataraman2024realworld}; and autonomous driving and navigation, which rely typically on recorded datasets rather than unrestricted on-policy exploration \citep{kiran2022driving,shi2021autonomous,shah2023revind}.

The core challenge of offline RL is to optimize reward while avoiding actions that are unsupported by the data distribution \citep{lange2012batch,levine2020offline}. Offline RL algorithms address this challenge by shaping either the policy or the value-learning objective. Behavior Regularized Actor Critic (BRAC) approaches keep the learned policy close to the dataset behavior \citep{fujimoto2021td3bc,tarasov2023rebrac}; Conservative Q Learning (CQL) penalizes high Q-values on unlikely actions to reduce overestimation under dataset shift \citep{kumar2020cql}; and in-sample or implicit methods avoid querying out-of-distribution actions (e.g., IQL) \citep{kostrikov2022iql}. 
Separately, motivated by generative models' success, recent work replaces parametric policies with generative ones to capture multi-modal behavior. These include sequence-modeling policies such as Decision Transformer that predict actions conditioned on return \citep{chen2021decisiontransformer} and diffusion or flow-based policies such Decision Diffuser \citep{janner2022diffuser} and Flow Q learning \citep{park2025flow} that generate actions or trajectories by iterative denoising or flow dynamics. For a more detailed discussion on related works, see Appendix \ref{sec:related}.


Existing offline RL methods typically produce a single policy trained on the offline dataset and then deploy it \emph{as-is} at inference time. These methods do not explicitly exploit inference-time information about the particular states encountered by the agent beyond using it as the input to the pre-trained policy. 
Recent work on large language models and reasoning suggests that additional inference-time computation can improve performance on a given test instance. For example, self-consistency and search-based reasoning methods use extra computation at inference time to refine or select among candidate outputs \citep{wang2022selfconsistency,yao2023tree,muennighoff2025s1}, while more recent work has begun to explore reinforcement-learning-based approaches for inference-time improvement \citep{huang2025t1,li2025ttrl}. This motivates the question: whether, in offline RL, additional inference-time computation can similarly improve the performance of an agent.


In offline RL, realizing this idea requires identifying what form of inference-time computation can actually improve the deployed policy. An offline learner returns a pre-trained policy $\pi_\psi$ and an associated critic $Q_\phi$, intended to approximate the optimal policy and $Q$-values. When the policy is deployed without further adaptation, performance depends on how accurate these approximations are in the states and actions \emph{encountered at inference time}, and accurate $Q$ estimation is itself challenging for two reasons. First, the $Q$-value is intrinsically difficult to estimate because it represents a long-horizon quantity: it aggregates future rewards over many steps starting from a given state-action pair. Second, in offline RL the critic must approximate the $Q$-value of the \emph{optimal} policy, even though the available data might be generated by a different behavior policy. Beyond the difficulty of $Q$-estimation, the deployed policy is itself information-limited: it is the best policy extractable from the offline dataset under the chosen learning objective. Inference-time gains must therefore come from \emph{new information} that was not available during offline training: the particular state $s_t$ that the agent encounters at deployment. The information becomes useful only if it can be coupled with a model of the consequences of acting from $s_t$, allowing the agent to look ahead from this specific state and refine its decision based on what would happen next.

In the Markovian model, the one-step transition dynamics $P(s' \mid s,a)$ and reward function $r(s,a)$ are \textit{local objects}: they depend only on the current state-action pair, rather than on long-horizon estimates, and are independent of whatever policy is being deployed.  This fact motivates a different approach: learn world models of $P$ and $r$ from offline data and use these learned functions at inference time to formulate policy. The link between $P$ and $r$ and a corresponding optimal policy is provided by Model Predictive Control (MPC) \citep{CamachoBordons2007MPC,RawlingsMayneDiehl2017MPC}. 
Conventional MPC optimizes action sequences over rolled-out short-horizon trajectories defined by the function $P$, thus implicitly generating a policy.
We use MPC in a different way, using information gathered from the rollout simulations to update the parameters that define our policy, and not to merely find the best action sequence.
In this way, the initial policy learned from the offline data is updated continuously during deployment as new information is gathered about the system. See Figure~\ref{fig:mpc_dwm_flow} for a visual overview and Section~\ref{sec:MPCwDWM} for a detailed description of our method. We now outline our main contributions.
\begin{greenboxheadless}
\begin{squishenumerate}
    \item We introduce \emph{MPC with a Differentiable World Model} (MPCwDWM), an inference-time adaptation framework for offline RL. The world model consists of a learned generative state-transition sampler, a reward model, and a terminal-value function, all differentiable in the policy parameters (Section~\ref{sec:worldModel}).
    \vspace{0.9ex}
    \item Building on this pipeline, we perform inference-time \emph{policy-parameter} updates by backpropagating through imagined rollouts to refine the baseline policy before each action. This contrasts with prior MPC and gradient-based-planning methods that optimize action sequences while leaving the policy fixed (Section~\ref{sec:MPCwDWM}).
    \vspace{0.9ex}
    \item We evaluate on D4RL continuous-control benchmarks~\citep{fu2020d4rl} (18 MuJoCo locomotion and 6 AntMaze datasets), showing consistent gains over strong offline RL baselines (Section~\ref{sec:exp}).
    \vspace{0.9ex}
    \item To address the compute requirements of diffusion unrolling and BPTT, we propose \emph{FlowMPC}, a one-step MeanFlow variant of the transition sampler that recovers much of the gain at a fraction of the compute (Section~\ref{sec:flowmpc}).

\end{squishenumerate}
\end{greenboxheadless}

\textbf{Comparison with existing World Model based Offline RL methods.} Although world models (including diffusion-based world models) have been used in \emph{offline} RL, they are typically leveraged in two ways. Some methods \citep{kidambi2020morel,yu2020mopo,yu2021combo,ding2024dwm} use the learned dynamics primarily during \emph{training}, generating imagined rollouts from the offline dataset to construct additional targets or synthetic experience for policy and value learning, including diffusion world models that model multi-step futures without step-by-step rollout. Other methods use a generative model at \emph{inference time} to produce candidates for future {optimal state trajectories} via sampling (often with guidance via return-conditioning), and then select an action by executing the first action of a sampled plan, without adapting the underlying policy parameters during inference \citep{ajay2023conditionalgen,janner2022diffuser,ki2025priorguided,yun2024gtg}. In contrast, our method uses the current observed state to \emph{adapt the policy parameters at inference time} by backpropagating through a differentiable world model over imagined finite-horizon rollouts.

\textbf{Comparison with MPC-based methods.} Although MPC has been combined with learned dynamics in several recent works, these methods either operate online with environment interaction or perform gradient-free planning over action sequences. Latent-world-model methods such as Dreamer \citep{hafner2020dreamer,hafner2021dreamerv2, hafner2025dreamerv3} and TD-MPC \citep{hansen2022tdmpc,hansen2024tdmpc2} train jointly with online interaction; the former uses the model only at training time, while the latter performs gradient-free MPPI in the latent space at deployment. Gradient-based and embedded differentiable MPC methods \citep{jyothir2023gradient,romero2025acmpc} optimize action sequences or cost-function parameters rather than the policy itself. Diffusion-based MPC methods such as D-MPC \citep{zhou2024dmpc} sample candidate action sequences from a learned diffusion proposal and select the best by gradient-free random shooting under a multi-step diffusion dynamics model. In contrast, our method operates purely on offline data and uses the current observed state to \emph{adapt the policy parameters at inference time} by backpropagating through a differentiable world model over imagined finite-horizon rollouts. A detailed methodological comparison against each of these is provided in Appendix~\ref{sec:mpc_baselines}.

\section{Preliminaries}

We consider a discounted Markov decision process (MDP) specified by the tuple
$\mathcal{M}=\langle \mathcal{S},\mathcal{A},P,d_0,r,\gamma\rangle$,
where $\mathcal{S}$ and $\mathcal{A}$ are the state and action spaces, $P(s'\mid s,a)$ is the transition kernel, $d_0$ is the initial state distribution, $r(s,a)$ is the reward function, and $\gamma\in[0,1)$ is the discount factor.
A stochastic policy $\pi(a\mid s)$ induces a distribution over trajectories $\tau=(s_0,a_0,r_0,s_1,a_1,r_1,\ldots)$ with density
$p^\pi(\tau)=d_0(s_0)\prod_{t\ge 0}\pi(a_t\mid s_t)P(s_{t+1}\mid s_t,a_t)$
and discounted return $R(\tau) \;=\; \sum_{t=0}^{\infty}\gamma^t r(s_t,a_t).$
The RL objective is to find a return-maximizing policy within a policy class $\Pi$,
$\pi^{\star} \in \argmax_{\pi\in\Pi} ~~\mathbb{E}_{\tau\sim p^\pi}\!\left[R(\tau)\right]~.$

\textbf{Temporal-difference learning.}
TD methods approximate the optimal action-value function
$Q^{\star}(s,a)=\mathbb{E}_{\tau\sim p^{\pi^{\star}}}[R(\tau)\mid s_0=s,a_0=a]$
with a parameterized critic $Q_\phi$ by minimizing a Bellman error on the data
$$\mathcal{L}_{\mathrm{TD}}(\phi)
=\!\!\!\!\!\!
\mathop{\mathbb{E}}\limits_{\substack{(s,a,r,s')\sim \mathcal{D}}}
\Big[
\big(r + \gamma \max_{a'\in\mathcal{A}} Q_{\bar\phi}(s',a') - Q_\phi(s,a)\big)^2
\Big]~,$$
where $Q_{\bar\phi}$ denotes a target network whose parameters $\bar\phi$ are a slowly updated copy of $\phi$.
For continuous action spaces, the maximization is typically implemented via a parameterized actor
$\pi_\psi(a\mid s)$, leading to a policy objective of the form
$\mathcal{J}(\psi)
\;:=\;
\mathbb{E}_{s\sim \mathcal{D},\,a\sim \pi_\psi(\cdot\mid s)}\!\left[Q_\phi(s,a)\right]~.$

\textbf{Offline RL and distribution shift.}
Offline RL learns a policy from a fixed dataset without additional environment interactions.
We assume access to a static dataset of transitions
$\mathcal{D}=\{(s_t,a_t,r_t,s_{t+1})\}_{t=1}^{N}$ collected by an unknown behavior policy $\mu$.
A core difficulty is that naively applying TD learning and policy improvement can drive the learned policy
$\pi_\psi$ toward state-action regions that are weakly represented in $\mathcal{D}$, i.e., the induced
discounted state-action occupancy measure $d^{\pi_\psi}$ moves away from $d^{\mu}$, where
$d^{\pi}(s,a)\;:=\;(1-\gamma)\sum_{t\ge 0}\gamma^t \Pr_{\tau\sim p^{\pi}}\!(s_t=s,\,a_t=a)$.
Many offline RL methods address this distribution shift by enforcing an explicit constraint,
either at the policy level
$D\!\left(\pi_\psi(\cdot\mid s)\,\middle\|\,\mu(\cdot\mid s)\right)\ \le\ \varepsilon,$
or at the occupancy level
$D\!\left(d^{\pi_\psi}\,\middle\|\,d^{\mu}\right)\ \le\ \varepsilon,$
where $D$ is a divergence or discrepancy measure.
Such constrained formulations often introduce additional algorithmic heuristics to obtain stable and competitive performance in practice.

\label{sec:prelim}

\section{Components of the World Model}
\label{sec:worldModel}


Our algorithm at inference time requires sampling next state samples conditioned on the current state action pair, and differentiating those samples with respect to the conditioning variables. We refer to any model with this sample plus gradient interface as a differentiable generative world model (DiffGenWM), and in this work we implement it using a conditional generative model. In addition to the DiffGenWM module, our world model includes a differentiable reward model and a terminal value function given by the critic of a pretrained policy. Concretely, we build a world model using the offline dataset of trajectories $\mathcal{D} = \{(s_t,a_t,r_t,s_{t+1})\}_{t=1}^N$, consisting of (1) a \emph{differentiable diffusion sampler} $f_{\theta}$ to simulate the transition dynamics, (2) a reward model $r_{\xi}$ to learn the reward of a given state action pair and, (3) pretrained policy $\pi_{\psi}$ and the corresponding pretrained critic value $Q_{\phi}$.

\subsection{Differentiable Diffusion Sampler}

We learn a parametric, differentiable sampler
$s_{t+1} = f_{\theta}(s_t, a_t, \varepsilon_t), \;\varepsilon_t \sim p_0,$
that represents a conditional distribution $p_{\theta}(s_{t+1}\mid s_t,a_t)$ over next states, with all stochasticity isolated in $\varepsilon_t$ via reparameterization. Here, \emph{differentiable} means in $(s_t,a_t)$ for fixed $\varepsilon_t$, not in $\theta$. We train $f_\theta$ on the offline dataset $\mathcal{D}=\{(s_t,a_t,r_t,s_{t+1})\}_{t=1}^N$ as a conditional diffusion model \citep{ho2020denoising}.

\textbf{Forward process.} Given a horizon $K$ and a schedule $\{\alpha_k\}_{k=1}^K\subset(0,1)$ with $\bar\alpha_k:=\prod_{j=1}^k \alpha_j$, we corrupt $s_{t+1}$ via a Gaussian Markov chain whose closed-form marginal at level $k$ is
$s_{t+1}^{(k)} = \sqrt{\bar\alpha_k}\, s_{t+1} + \sqrt{1-\bar\alpha_k}\,\epsilon$, where $\epsilon\sim\mathcal{N}(0,I)$.

\textbf{Reverse process.} Conditioned on $c_t:=(s_t,a_t)$, we initialize $s_{t+1}^{(K)} = z_K$ with $z_K\sim\mathcal{N}(0,I)$ and iteratively denoise:
\begin{align*}
s_{t+1}^{(k-1)} = \tfrac{1}{\sqrt{\alpha_k}}\!\left(s_{t+1}^{(k)} - \tfrac{1-\alpha_k}{\sqrt{1-\bar\alpha_k}}\,\hat\epsilon_{\theta}(s_{t+1}^{(k)},k,c_t)\right) + \sigma_k\, z_{k-1}, \quad z_{k-1}\sim\mathcal{N}(0,I),
\end{align*}
where $\hat\epsilon_\theta$ is a learned noise predictor and $\sigma_k^2 = 1-\alpha_k$. Collecting the initial draw and the reverse-step Gaussians into $\varepsilon_t:=(z_K,z_{K-1},\dots,z_0)$, with $p_0$ the product of $K+1$ standard Gaussians, the sampler $(s_t,a_t)\mapsto s_{t+1}^{(0)}$ is deterministic for fixed $\varepsilon_t$ and is therefore differentiable in $(s_t,a_t)$.

\textbf{Training objective.} We minimize the standard conditional denoising loss
$$\mathcal{L}_{\mathrm{diff}}(\theta) = \mathbb{E}_{(s_t,a_t,s_{t+1})\sim\mathcal{D},\,k\sim\mathrm{Unif}\{1,\ldots,K\},\,\epsilon\sim\mathcal{N}(0,I)}\!\left[\|\epsilon - \hat\epsilon_\theta(s_{t+1}^{(k)},k,c_t)\|_2^2\right],$$ with $s_{t+1}^{(k)}$ given by the forward marginal above. The trained $f_\theta$ serves as a learned, differentiable one-step dynamics simulator; see Appendix~\ref{sec:dds} for the full construction, including the explicit initialization and reverse-step maps.

\subsection{Reward model}
In addition to the transition model, we learn a parametric reward predictor from the same offline data. We parameterize the reward predictor as a function
$r_{\xi} : \mathcal{S}\times\mathcal{A}\to\mathbb{R}$,
which maps a state--action pair $(s_t,a_t)$ to a scalar reward estimate. Given an offline dataset of transitions $\mathcal{D}=\{(s_t,a_t,r_t,s_{t+1})\}_{t=1}^N$, we fit $r_{\xi}$ by supervised regression on observed rewards, minimizing the loss function
$$\mathcal{L}_{r}(\xi)
:= \underset{(s_t,a_t,r_t)\sim\mathcal{D}}{\mathbb{E}}
\Big[
\big(r_{\xi}(s_t,a_t)-r_t\big)^2
\Big].$$
After training, $r_{\xi}$ serves as a differentiable reward oracle that can be queried at arbitrary $(s,a)$ pairs produced by downstream planning or policy optimization, and can be combined with the learned diffusion transition model to form multi-step return objectives.

\subsection{Policy and Terminal Value}

Given an offline dataset $\mathcal{D}=\{(s_t,a_t,r_t,s_{t+1})\}_{t=1}^N$, we learn a parametric stochastic policy $\pi_{\psi}(a\mid s)$ and a parametric action-value function $Q_{\phi}(s,a)$ via a BC-regularized actor-critic, in the spirit of BRAC \citep{tarasov2023minimalist,wu2019brac}. The critic minimizes a one-step temporal-difference loss with target network $Q_{\bar{\phi}}$,
$$\mathcal{L}_{Q}(\phi)
:=
\mathbb{E}_{\substack{(s,a,r,s')\sim \mathcal{D}\\ a'\sim \pi_{\psi}(\cdot\mid s')}}
\Big[
\big(Q_{\phi}(s,a) - r - \gamma Q_{\bar{\phi}}(s',a')\big)^2
\Big].$$
The actor maximizes value under $Q_\phi$ while staying close to the dataset behavior via a behavior-cloning term that increases the likelihood of dataset actions under $\pi_{\psi}$,
$$\mathcal{L}_{\pi}(\psi)
:=
\mathbb{E}_{\substack{(s,a)\sim \mathcal{D}\\ a^{\pi}\sim \pi_{\psi}(\cdot\mid s)}}
\Big[
- Q_{\phi}(s,a^{\pi})
\;-\;
\lambda \log \pi_{\psi}(a\mid s)
\Big],$$
where $\lambda>0$ controls the strength of the regularization. In practice, $Q_{\bar{\phi}}$ is maintained as a slowly updated copy of $Q_{\phi}$ to stabilize the bootstrapped target in $\mathcal{L}_{Q}(\phi)$.

\section{MPC with Differentiable World Model}
\label{sec:MPCwDWM}
\begin{figure}[t]
    \centering
    \caption{Inference-time MPC with a diffusion world model. An offline dataset trains (i) a diffusion sampler $f_\theta$, (ii) a reward model $r_\xi$, and (iii) a policy $\pi_\psi$ and terminal critic $Q_\phi$. At state $s_t$, we unroll $M$ imagined rollouts via $\tilde a_h = \pi_\psi(\tilde s_h)$ and $\tilde s_{h+1} = f_\theta(\tilde s_h, \tilde a_h, \varepsilon_{t+h})$, score them with predicted rewards plus a terminal value, and backpropagate through the rollout to update $\psi$ before executing the first action. {\color{Green} Green: forward rollout}; {\color{Brown}Brown: gradient flow}.}
    \includegraphics[width=0.95\linewidth]{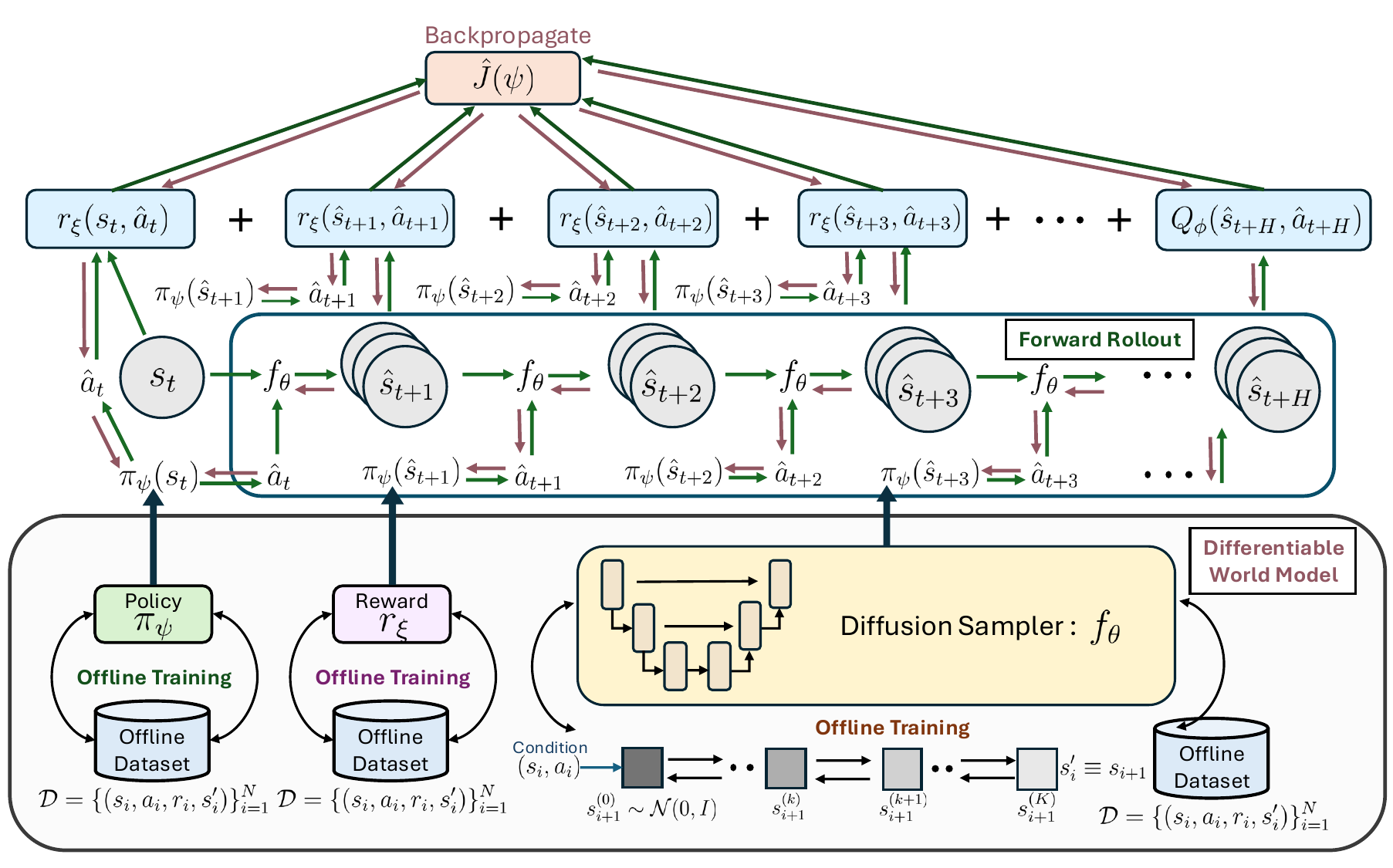}
\label{fig:mpc_dwm_flow}
\end{figure}

We now describe our receding-horizon model predictive control (MPC) procedure that uses the learned diffusion world model as a differentiable simulator. The world model is given by the reparameterized transition map $s_{t+1} = f_{\theta}(s_t,a_t,\varepsilon_t), \varepsilon_t \sim p_0(\varepsilon)$, where $f_{\theta}$ denotes the reverse diffusion sampler unrolled as a computation graph. We combine this dynamics model with a learned reward predictor $r_{\xi}(s,a)$ and a learned terminal critic $Q_{\phi}(s,a)$. The critic provides a terminal value for truncated rollouts, while $r_{\xi}$ supplies the immediate reward along imagined trajectories at inference time. Without loss of generality, we assume $\pi_\psi$ is deterministic for the remainder of this section.

\textbf{Receding-horizon Objective:}
At a real environment state $s_t$, MPC optimizes a horizon-$H$ objective over imagined trajectories generated by $f_{\theta}$. Given a fixed noise sequence
$\varepsilon_{t:t+H-1}=(\varepsilon_t,\ldots,\varepsilon_{t+H-1})$
define the imagined rollout recursively by
$\tilde{s}_0 = s_t,
\,\,
\tilde{a}_j = \pi_{\psi}(\tilde{s}_j),
\,\,
\tilde{s}_{j+1} = f_{\theta}(\tilde{s}_j,\tilde{a}_j,\varepsilon_{t+j})$, where $j=0,\ldots,H-1$.
We evaluate the finite-horizon return by summing predicted stage rewards and adding a terminal value given by the critic:
\begin{align*}
L(\psi;\varepsilon_{t:t+H-1})
&=
\sum_{j=0}^{H-1}\gamma^j\,
r_{\xi}(\tilde{s}_j,\tilde{a}_j)
+
\gamma^H\,
Q_{\phi}\!\Big(
\tilde{s}_H,\pi_{\psi}(\tilde{s}_H)
\Big).
\end{align*}
We then define the MPC objective as the expectation of this return over diffusion noise: $$J_t(\psi)
=
\mathbb{E}_{\varepsilon_{t:t+H-1}\sim p_0}
\Big[
L(\psi;\varepsilon_{t:t+H-1})
\Big].$$

\textbf{Monte Carlo Approximation:} We approximate the objective $J_t(\psi)$ with $M$ i.i.d.\ noise sequences
$\{\varepsilon_{t:t+H-1}^{(m)}\}_{m=1}^{M}$, where
$\varepsilon_{t:t+H-1}^{(m)}\sim p_0$. For each $m$, generate an imagined rollout
$\{\tilde{s}^{(m)}_j,\tilde{a}^{(m)}_j\}_{j=0}^{H}$:
\begin{align*}
\tilde{s}_0^{(m)} = s_t, \,
\tilde{a}_j^{(m)} = \pi_{\psi}\!(\tilde{s}_j^{(m)}),\,
\tilde{s}_{j+1}^{(m)} =
f_{\theta}\!(
\tilde{s}_j^{(m)},\tilde{a}_j^{(m)},\varepsilon_{t+j}^{(m)}
).
\end{align*}
The Monte Carlo estimator is
\begin{align*}
\widehat{J}_t(\psi)
&=
\frac{1}{M}\sum_{m=1}^{M}
\Bigg[
\sum_{j=0}^{H-1}\gamma^j\,
r_{\xi}\!\big(\tilde{s}_j^{(m)},\tilde{a}_j^{(m)}\big) +
\gamma^H\,
Q_{\phi}\!\Big(
\tilde{s}_H^{(m)},\pi_{\psi}(\tilde{s}_H^{(m)})
\Big)
\Bigg].
\end{align*}
\textbf{Gradient-based optimization:}
At time $t$, we perform $E$ steps of gradient ascent on $\widehat{J}_t(\psi)$ while holding
$f_{\theta}$, $r_{\xi}$, and $Q_{\phi}$ fixed:
$\psi \leftarrow \psi + \eta\, \nabla_{\psi}\widehat{J}_t(\psi)$,
where $\eta>0$ is the step size. After the inner loop, MPC executes the first action in the real environment $a_t = \pi_{\psi}(s_t)$, observes $(r_t,s_{t+1})$, and repeats the procedure at the next state $s_{t+1}$. See Figure~\ref{fig:mpc_dwm_flow} for a visual overview and Algorithm~\ref{alg:mpc_diffusion_world_model} for detailed steps. For a closed-form gradient recursion for our setup see Appendix~\ref{sec:mpc_backprop_diffusion_critic}.

\begin{algorithm}[!t]
\caption{Model Predictive Control with Differentiable World Model (MPCwDWM)}
\label{alg:mpc_diffusion_world_model}
\begin{algorithmic}[1]
\Require Diffusion sampler $f_{\theta}$, noise $\{\varepsilon_{t:t+H-1}^{(i)}\}_{i=1}^M$ reward model $r_{\xi}$, terminal critic $Q_{\phi}$, policy $\pi_{\psi}$, horizon $H$, particles $M$, inner steps $E$, step size $\eta$, discount $\gamma$.
\For{$t=1,2,3\ldots, T$}
\State Observe environment state $s_t$
\For{$e=1,2,\dots,E$}
    \State $J \gets 0$, \; $\tilde{s}_0^{(i)} \gets s_t \quad \forall\, i \in \{1,2,\dots,M\}$
    \For{$j=0,1,\dots,H-1$}
        \For{$i=1,2,\dots,M$}
            \State $\tilde{a}_j^{(i)} \gets \pi_{\psi} (\tilde{s}_j^{(i)})$ \algc{\color{Green} Action for imagined rollout}
        \EndFor
        \State $J \gets J + \frac{1}{M}\sum_{i=1}^M \gamma^j\, r_{\xi}(\tilde{s}_j^{(i)},\tilde{a}_j^{(i)})$ \label{line:mpc_reward_accum}
        \For{$i=1,2,\dots,M$}
            \State $\tilde{s}_{j+1}^{(i)} \gets f_{\theta}(\tilde{s}_j^{(i)},\tilde{a}_j^{(i)},\varepsilon_{t+j}^{(i)})$ \algc{\color{Green} Next state}
        \EndFor
    \EndFor
    \State $J \gets J + \frac{1}{M}\sum_{i=1}^M \gamma^H\, Q_{\phi}(\tilde{s}_H^{(i)},\pi_{\psi}(\tilde{s}_H^{(i)}))$ \label{line:mpc_terminal}
    \State $\psi \gets \psi + \eta\, \nabla_{\psi} J$ \algc{\color{Green} Policy Update}
    \label{line:mpc_update}
\EndFor
\State Execute $a_t \gets \pi_{\psi}(s_t)$ and observe $(r_t,s_{t+1})$ \label{line:mpc_execute}
\EndFor
\end{algorithmic}
\end{algorithm}

\section{Experiments}
\label{sec:exp}

\textbf{Datasets.}
We evaluate our algorithm on the D4RL benchmark \citep{fu2020d4rl}, focusing on continuous-control domains where offline learning and long-horizon credit assignment are challenging. Our main results cover Gym MuJoCo locomotion datasets (15 tasks spanning medium, expert, medium-expert, medium-replay, and full-replay settings across all three environments HalfCheetah, Hopper, and Walker2d). We also consider all AntMaze datasets (6 tasks spanning maze, maze-diverse, medium-play, medium-diverse, large-play, and large-diverse). 

\textbf{Baselines.}
We compare against strong offline RL baselines that are widely used on D4RL locomotion. These include TD3+BC \citep{fujimoto2021td3bc}, which adds behavior cloning regularization to a TD3 actor update; IQL \citep{kostrikov2022iql}, which performs offline policy improvement via value-weighted regression without explicit behavior constraints; CQL \citep{kumar2020cql}, which trains a conservative critic by penalizing high values on out-of-distribution actions; 
and ReBRAC \citep{tarasov2023rebrac}, which revisits behavior-regularized actor-critic with practical improvements and strong performance. 

\begin{table}[!t]
\centering
\normalsize
\resizebox{\linewidth}{!}{%
\begin{tabular}{l|ccc|cc}
\toprule
Dataset & TD3+BC & IQL & CQL & {ReBRAC} & $\DiffMPC$ \\
\midrule
\multicolumn{6}{c}{\textbf{halfcheetah}} \\
\midrule
medium        & $54.7 \pm 0.9$ & $50.0 \pm 0.2$ & $46.9 \pm 0.4$ & $65.87 \pm 1.3$  & $\textbf{70.05} \pm 1.8$ \\
expert        & $93.4 \pm 0.4$ & $95.5 \pm 2.1$ & $97.3 \pm 1.1$ & $105.22 \pm 0.92$ & $\textbf{106.14} \pm 2.1$ \\
medium-expert & $89.1 \pm 5.6$ & $92.7 \pm 2.8$ & $95.0 \pm 1.4$ & $100.29 \pm 8.2$ & $\textbf{105.35} \pm 1.8$ \\
medium-replay & $45.0 \pm 1.1$ & $42.1 \pm 3.6$ & $45.3 \pm 0.3$ & $50.21 \pm 0.7$  & $\textbf{59.89} \pm 1.2$ \\
full-replay   & $75.0 \pm 2.5$ & $75.0 \pm 0.7$ & $76.9 \pm 0.9$ & $82.19 \pm 0.9$  & $\textbf{85.149} \pm 0.9$ \\
\midrule
\multicolumn{6}{c}{\textbf{hopper}} \\
\midrule
medium        & $60.9 \pm 7.6$  & $65.2 \pm 4.2$ & $61.9 \pm 6.4$ & $102.05 \pm 0.9$ & $\textbf{103.38} \pm 0.35$ \\
expert        & $\textbf{109.6} \pm 3.7$ & $108.8 \pm 3.1$ & $106.5 \pm 9.1$ & $104.73 \pm 6.94$ & $105.06 \pm 5.02$ \\
medium-expert & $87.8 \pm 10.5$ & $85.5 \pm 29.7$ & $96.9 \pm 15.1$ & $105.70 \pm 8.3$ & $\textbf{107.42} \pm 5.30$ \\
medium-replay & $55.1 \pm 31.7$ & $89.6 \pm 13.2$ & $86.3 \pm 7.3$  & $95.82 \pm 7.9$  & $\textbf{103.14} \pm 0.45$ \\
full-replay   & $97.9 \pm 17.5$ & $104.4 \pm 10.8$ & $101.9 \pm 0.6$ & $107.49 \pm 0.4$ & $\textbf{108.77} \pm 0.6434$ \\
\midrule
\multicolumn{6}{c}{\textbf{walker2d}} \\
\midrule
medium        & $77.7 \pm 2.9$  & $80.7 \pm 3.4$  & $79.5 \pm 3.2$  & $84.78 \pm 1.8$  & $\textbf{88.91} \pm 0.6$ \\
expert        & $110.0 \pm 0.6$ & $96.9 \pm 32.3$ & $109.3 \pm 0.1$ & $112.15 \pm 0.2$ & $\textbf{116.65} \pm 0.41$ \\
medium-expert & $110.4 \pm 0.6$ & $112.1 \pm 0.5$ & $109.1 \pm 0.2$ & $111.81 \pm 0.3$ & $\textbf{115.76} \pm 1.04$ \\
medium-replay & $68.0 \pm 19.2$ & $75.4 \pm 9.3$  & $76.8 \pm 10.0$ & $80.67 \pm 5.8$  & $\textbf{95.87} \pm 1.19$ \\
full-replay   & $90.3 \pm 5.4$  & $97.5 \pm 1.4$  & $94.2 \pm 1.9$  & $100.16 \pm 8.8$ & $\textbf{105.76} \pm 2.91$ \\
\midrule
Average & $81.7$ & $84.8$ & $85.6$ & $93.9$ & $\textbf{98.5}$ \\
\bottomrule
\end{tabular}%
}
\vspace{2ex}
\caption{\normalsize Average normalized score over the final evaluation and ten unseen training seeds on Gym-MuJoCo tasks. TD3+BC, IQL, and CQL scores are taken from \citep{tarasov2023minimalist}. We ran ReBRAC and $\DiffMPC$ with ReBRAC as the pre-trained policy. The symbol $\pm$ represents the standard deviation across seeds.
}
\label{tab:rebrac_mujoco}
\end{table}

\begin{wraptable}{r}{0.44\textwidth}
\centering
\renewcommand{\arraystretch}{1.05}
\setlength{\tabcolsep}{2pt}
\begin{tabular*}{\linewidth}{@{\extracolsep{\fill}}l|cc@{}}
\toprule
Task Name & ReBRAC & $\DiffMPC$ \\
\midrule
antmaze-umaze          & $\textbf{99.0} \pm 3.0$  & $\textbf{99.0} \pm 3.0$ \\
umaze-diverse  & $89.0 \pm 7.0$ & $\textbf{90.0} \pm 6.3$ \\
medium-play    & $86.7 \pm 11.5$  & $\textbf{94.4} \pm 6.8$ \\
medium-diverse & $83.0 \pm 10.1$ & $\textbf{94.0} \pm 8.0$ \\
large-play     & $59.0 \pm 25.5$ & $\textbf{67.0} \pm 26.1$ \\
large-diverse  & $51.0 \pm 27.0$ & $\textbf{66.0} \pm 30.4$ \\
\midrule
Average & $77.62$ & $\textbf{85.07}$ \\
\bottomrule
\end{tabular*}
\caption{Average normalized score on ten unseen seeds on \textbf{AntMaze}. $\pm$ represents the standard deviation across seeds.}
\label{tab:rebrac_antmaze}
\end{wraptable}
\textbf{Experimental details.}
Our approach has an offline training stage and an inference-time adaptation stage. We first train a policy $\pi_{\psi}$ and critic $Q_{\phi}$ on the offline dataset using ReBRAC \citep{tarasov2023rebrac}. We then train a diffusion transition model $f_{\theta}$ on dataset transitions and a reward predictor $r_{\xi}$ on dataset rewards using supervised learning. Since the inference-time adaptation differentiates through the diffusion sampler, we refer to this instantiation of our method as $\DiffMPC$.
Training details and results for both models are deferred to Appendix~\ref{sec:exp_details}. 

To train the ReBRAC policy we follow the hyper-parameter choices from \citep{tarasov2023minimalist}. We tune the inference time parameters on a single training seed (different from the evaluation seeds). For a detailed description and hyper-parameter choices see Appendix~\ref{sec:exp_details}

\textbf{Results.} Across D4RL \citep{fu2020d4rl} continuous-control tasks, $\DiffMPC$ outperforms existing strong offline RL baselines (see Table~\ref{tab:rebrac_mujoco}). It achieves average normalized scores of $98.5$ on Gym MuJoCo (vs. ReBRAC $93.9$). On all datasets, our method $\DiffMPC$ improves the performance from the pre-trained ReBRAC policy, showing that inference time adaptation of the policy brings statistical benefits. 
We also compare $\DiffMPC$ with the pre-trained ReBRAC (outperforms TD3+BC, IQL, CQL and SAC-RND  in \citep{tarasov2023minimalist}) on D4RL Antmaze task in Table~\ref{tab:rebrac_antmaze} and show that inference time adaptation improves the performance. Our method acheives an average score of $85.07$ on AntMaze (vs.\ ReBRAC $77.62$), with the largest gains on hard AntMaze variants (large-play $59\rightarrow 67$, large-diverse $51\rightarrow66$).

\begin{table}[t]
\centering
\resizebox{\linewidth}{!}{%
\begin{tabular}{lrrrrrrr}
\toprule
Dataset & DT & TT & MOPO & MOReL & MBOP & Diffuser & $\DiffMPC$ \\
\midrule
\multicolumn{8}{c}{\textbf{HalfCheetah}} \\
\midrule
medium-expert & 86.8 & 95.0 & 63.3 & 53.3 & \textbf{105.9} & $88.9 \pm 0.3$ & $105.35 \pm 1.8$ \\
medium        & 42.6 & 46.9 & 42.3 & 42.1 & 44.6 & $42.8 \pm 0.3$ & $\textbf{70.05} \pm 1.8$ \\
medium-replay & 36.6 & 41.9 & 53.1 & 40.2 & 42.3 & $37.7 \pm 0.5$ & $\textbf{59.89} \pm 1.2$ \\
\midrule
\multicolumn{8}{c}{\textbf{Hopper}} \\
\midrule
medium-expert & 107.6 & \textbf{110.0} & 23.7 & 108.7 & 55.1 & $103.3 \pm 1.3$ & $107.420 \pm 5.30$ \\
medium        & 67.6 & 61.1 & 28.0 & 95.4 & 48.8 & $74.3 \pm 1.4$ & $\textbf{103.38} \pm 0.35$ \\
medium-replay & 82.7 & 91.5 & 67.5 & 93.6 & 12.4 & $93.6 \pm 0.4$ & $\textbf{103.14} \pm 0.45$ \\
\midrule
\multicolumn{8}{c}{\textbf{Walker2d}} \\
\midrule
medium-expert & 108.1 & 101.9 & 44.6 & 95.6 & 70.2 & $106.9 \pm 0.2$ & $\textbf{115.76} \pm 1.04$ \\
medium        & 74.0 & 79.0 & 17.8 & 77.8 & 41.0 & $79.6 \pm 0.55$ & $\textbf{88.91} \pm 0.6$ \\
medium-replay & 66.6 & 82.6 & 39.0 & 49.8 & 9.7 & $70.6 \pm 1.6$ & $\textbf{95.87} \pm 1.19$ \\
\midrule
Average & 74.7 & 78.9 & 42.1 & 72.9 & 47.8 & 77.5 & \textbf{94.4} \\
\bottomrule
\end{tabular}%
}
\vspace{1ex}
\caption{Normalized scores on D4RL MuJoCo tasks. DT, TT, MOPO, MOReL, MBOP, and Diffuser values are taken from \cite{janner2022planning}. The symbol $\pm$ represents the standard deviation across the seeds.
}
\label{tab:selected_methods_mujoco_mpcwdwm}
\end{table}

\textbf{Generative Model Baselines.} We also compare against generative model based approaches to Offline RL. Decision Transformer (DT) casts offline RL as return conditioned sequence modeling and predicts actions with a Transformer conditioned on the desired return and history \cite{chen2021decisiontransformer}.
Trajectory Transformer (TT) models trajectories as sequences and plans by decoding high reward trajectories under the learned sequence model \cite{janner2021trajectorytransformer}.
MOPO learns a dynamics model from offline data and optimizes a policy in the learned model while penalizing rewards using model uncertainty to discourage exploiting model errors \cite{yu2020mopo}.
MOReL constructs a pessimistic surrogate MDP that transitions to an absorbing state in regions where the model is unreliable, then learns a policy within this conservative MDP \cite{kidambi2020morel}.
MBOP performs model predictive planning in a learned dynamics model while constraining candidate action sequences with a behavior prior and extending the horizon with a value estimate \cite{argenson2021mbop}.
Finally, Diffuser learns a diffusion model over trajectories and plans by iterative denoising with flexible conditioning, so sampling corresponds to generating high reward behavior consistent with the constraints \cite{janner2022diffuser}. 
Since prior generative-model baselines report MuJoCo results on a subset of D4RL tasks, we use the same evaluation subset for a direct comparison.
Table~\ref{tab:selected_methods_mujoco_mpcwdwm} shows that $\DiffMPC$ consistently outperforms prior generative model based offline RL approaches across all environments.
In particular, $\DiffMPC$ achieves the best average score of $94.4$, outperforming Diffuser at $77.5$.

\textbf{MPC-based baselines.}
We also compare against recent MPC-based approaches for continuous control. 
TD-MPC and TD-MPC2 use learned latent dynamics together with Model Predictive Path Integral (MPPI) control: at each step, candidate action sequences are sampled, rolled out through the latent dynamics, scored using reward and terminal value estimates, and iteratively refined into a sampling distribution from which the first action is executed \citep{hansen2022tdmpc,hansen2024tdmpc2}.
These methods are developed for online control, where the replay buffer is populated by ongoing environment interaction; we provide a detailed comparison with these latent-space MPC methods, together with TD-MPC2 result tables on D4RL Gym-MuJoCo, in Appendix~\ref{app:latent-world-models}. 
We further compare against gradient-based optimization of action sequences through learned differentiable rollouts \citep{jyothir2023gradient}; Appendix~\ref{app:gradient-mpc} gives the detailed design-space comparison and the full action-update ablation. In both cases, $\DiffMPC$ outperforms the corresponding baselines on offline D4RL Gym-MuJoCo datasets.

\begin{wraptable}{r}{0.58\textwidth}
\centering
\renewcommand{\arraystretch}{1.05}
\setlength{\tabcolsep}{2pt}
\begin{tabular*}{\linewidth}{@{\extracolsep{\fill}}l|cc@{}}
\toprule
Task Name & D-MPC & $\DiffMPC$ \\
\midrule
halfcheetah medium        & $46.00 \pm 0.17$ & $\textbf{70.05} \pm 1.8$ \\
hopper medium             & $61.24 \pm 2.30$ & $\textbf{103.38} \pm 0.35$ \\
walker2d medium           & $76.21 \pm 2.67$ & $\textbf{88.91} \pm 0.6$ \\
halfcheetah medium-replay & $41.12 \pm 0.31$ & $\textbf{59.89} \pm 1.2$ \\
hopper medium-replay      & $92.49 \pm 2.23$ & $\textbf{103.14} \pm 0.45$ \\
walker2d medium-replay    & $78.81 \pm 4.19$ & $\textbf{95.87} \pm 1.19$ \\
\midrule
Average & $65.98$ & $\textbf{86.87}$ \\
\bottomrule
\end{tabular*}
\caption{Average normalized score on ten unseen seeds. $\pm$ represents the standard deviation across seeds.\vspace{-2ex}}
\label{tab:dmpc_locomotion}
\end{wraptable}

Finally, we compare against D-MPC, a diffusion-based MPC method that trains a diffusion action proposal and a diffusion dynamics model from offline data, then performs gradient-free MPC by random shooting over sampled action sequences \citep{zhou2024dmpc}. 
D-MPC improves over prior model-based MPC baselines such as MBOP, but on the locomotion subset in its original comparison, its average score remains below other offline RL baselines (which $\DiffMPC$ outperforms). 
Unlike D-MPC, which selects the best sampled open-loop action sequence, $\DiffMPC$ updates the parameters of a pretrained policy by backpropagating through differentiable imagined rollouts inside the MPC loop. 
The detailed comparison is given in Appendix~\ref{app:dmpc}, 
Since D-MPC reports MuJoCo results on a subset of D4RL tasks, we use the same evaluation subset for a direct comparison. Table~\ref{tab:dmpc_locomotion} shows that $\DiffMPC$ achieves an average normalized score of $86.87$, compared with $65.98$ for D-MPC, and improves over D-MPC on all six tasks.

\textbf{Ablation.}
We conduct extensive ablations on the key hyperparameters of $\DiffMPC$: planning horizon $H$, inner gradient steps $E$, step size $\eta$, DDIM sampling steps $K$, and the number of Monte Carlo particles $M$, across all Gym-MuJoCo datasets (Appendix~\ref{sec:ablations}). The method is robust across a wide range: moderate horizons ($H \in \{5, 10\}$) and a single inner step ($E=1$) recover most of the gains from inference-time adaptation at modest cost, while $M$ has negligible effect within $\{512, 1024, 4096\}$. Larger $E$ and $\eta$ yield further improvements on stable environments such as Walker2d, but can amplify model and value errors on termination-sensitive tasks like Hopper, consistent with known failure modes of offline RL under aggressive policy updates. The DDIM ablation shows diminishing returns from additional denoising steps on most datasets, motivating the one-step MeanFlow sampler in Section~\ref{sec:flowmpc}. We additionally validate $\DiffMPC$ with SAC-RND as the pre-trained policy, where it improves the average normalized score from $84.08$ to $93.00$ (Appendix~\ref{sec:sac-rnd}).

\section{Inference-Time Compute and One-Step World Model}
\label{sec:flowmpc}
\begin{table}[ht]
\centering
\begin{minipage}[t]{0.48\linewidth}
\centering
\vspace{0pt}
\resizebox{\linewidth}{!}{%
\begin{tabular}{c|c|c}
\toprule
$H$ & Zero-shot (ms/step) & DiffMPC (ms/step) \\
\midrule
1 & $0.409 \pm 0.053$ & $48.42 \pm 0.98$ \\
2 & $0.409 \pm 0.053$ & $84.76 \pm 1.76$ \\
5 & $0.409 \pm 0.053$ & $194.96 \pm 14.51$ \\
\bottomrule
\end{tabular}%
}
\caption{Per-step wall-clock time on Hopper while sweeping planning horizon $H$ (DDIM denoising steps $K = 10$, inner gradient steps $E=1$, one episode of 1000 steps, mean $\pm$ std).\vspace{-1ex}}
\label{tab:hopper_h_sweep_time}
\end{minipage}
\hfill
\begin{minipage}[t]{0.48\linewidth}
\centering
\vspace{0pt}
\resizebox{\linewidth}{!}{%
\begin{tabular}{c|c|c}
\toprule
$K$ & Zero-shot (ms/step) & DiffMPC (ms/step) \\
\midrule
1 & $0.390 \pm 0.049$ & $57.36 \pm 1.51$ \\
5 & $0.390 \pm 0.049$ & $121.86 \pm 11.37$ \\
10 & $0.390 \pm 0.049$ & $198.77 \pm 24.64$ \\
\bottomrule
\end{tabular}%
}
\caption{Per-step wall-clock time on Hopper while sweeping DDIM steps $K$ (planning horizon $H=5$, inner gradient steps $E=1$, one episode of 1000 steps, mean $\pm$ std).\vspace{-1ex}}
\label{tab:hopper_k_sweep_time}
\end{minipage}
\end{table}

A practical limitation of $\DiffMPC$ is the cost of iterative diffusion sampling: each imagined transition requires $K$ sequential denoising steps, and inference-time MPC issues many such queries per step. Tables~\ref{tab:hopper_h_sweep_time} and~\ref{tab:hopper_k_sweep_time} quantify this overhead on Hopper full-replay. The zero-shot ReBRAC policy selects an action in roughly $0.4$ ms, whereas $\DiffMPC$ at $H=5$, $K=10$ takes nearly $200$ ms per step, almost three orders of magnitude slower. Both the planning horizon and the number of denoising steps contribute roughly linearly to this cost: increasing $H$ from $1$ to $5$ raises latency from $48$ to $195$ ms, and increasing $K$ from $1$ to $10$ raises it from $57$ to $199$ ms. Since $K$ multiplies the cost of every world-model call inside the rollout, it is the more attractive axis to improve.

This motivates collapsing the $K$-step denoising chain into a single forward pass. We replace the diffusion sampler with a one-step \emph{MeanFlow} model \citep{geng2025meanflows}, which maps $(s_t,a_t,\epsilon_t)$ to $s_{t+1}$ in a single forward pass while remaining differentiable. To offset the loss of expressiveness from collapsing many denoising steps into one, we additionally train the world model under a \emph{policy-tilted} objective that upweights state--action pairs likely under a strong pretrained policy, focusing capacity on the regions the deployed controller actually queries. We call the resulting method $\FlowMPC$.

On the D4RL Gym-MuJoCo suite, $\FlowMPC$ improves the average normalized score over the ReBRAC initialization from $94.93$ to $97.87$, with gains on all 15 datasets. The improvements are largest on the medium-quality datasets where there is the most room for inference-time policy refinement (e.g., halfcheetah-medium $65.42 \to 70.06$, walker2d-medium $84.79 \to 88.59$, halfcheetah-medium-replay $50.78 \to 56.51$), while remaining competitive with $\DiffMPC$ across most regimes. At the same time, $\FlowMPC$ reduces per-transition forward latency by roughly $8\times$ at $K=10$ and total forward+backward time by $15$--$75\times$ as $K$ grows. Full derivations of the MeanFlow training objective, the policy-tilted scheme, and complete benchmark and latency results are deferred to Appendix~\ref{sec:flowmpc_app}.

\paragraph{Scaling the planning horizon.} A second limitation of $\DiffMPC$ is that inference-time MPC cost scales linearly with horizon $H$. The natural fix is to replace the autoregressive rollout with a \emph{trajectory-level} generator producing $(s_{t+1},\dots,s_{t+H})$ in one forward pass, conditioned on the current state and policy. Several recent works pursue this (though not at inference time): \citet{rigter2024polygrad} train a diffusion model for on-policy trajectories via policy-gradient guidance, \citet{jackson2024policy} use policy-guided diffusion under a regularized target, and \citet{ajay2023conditionalgen} cast offline decision-making as conditional trajectory generation with classifier-free return guidance. The challenge in our setting is that $\pi_\psi$ is updated at inference time, so the generator must produce rollouts consistent with the \emph{current} $\pi_\psi$ without autoregressive unrolling. Possible approaches: policy-gradient guidance during sampling (à la PolyGRAD), or periodic fine-tuning of the generator with policy scores via RL-based methods \citep{black2024ddpo} or adjoint matching \citep{domingoenrich2025adjoint}. We view this as a promising path to amortize the per-step cost of long-horizon $\DiffMPC$, and leave a thorough exploration to future work.

Our primary goal in this work is to demonstrate the statistical benefits of inference-time policy optimization with differentiable world models. We take initial steps toward reducing its computational cost, through FlowMPC and the discussion of trajectory-level generation above, and leave a more thorough investigation to future work.

\section{Conclusion}
\label{sec:conclusion}

We introduced an inference-time adaptation framework for offline RL that refines a
pretrained policy on the fly using the current state and a learned differentiable world
model. The core design is a Differentiable World Model pipeline that supports end-to-end
gradient computation through imagined finite-horizon rollouts, enabling direct optimization
of policy parameters at deployment via a surrogate objective combining predicted rewards
and a terminal critic. Our diffusion instantiation, $\DiffMPC$, yields consistent gains over
strong offline RL baselines across D4RL Gym MuJoCo and AntMaze. Recognizing that iterative
diffusion sampling is the dominant inference-time cost, we proposed $\FlowMPC$, a one-step
MeanFlow variant trained under a policy-tilted objective, which recovers most of the gains
at a fraction of the wall-clock cost. Together, these results show that inference-time
computation, properly amortized, is a practical lever for offline RL. Promising directions
for future work include trajectory-level generative world models that produce finite-horizon
futures in a single pass, evaluations using generative policies (eg. Flow Q-Learning \citep{park2025flow}), and offline-to-online extensions with periodic world-model updates
to improve robustness under distribution shift.

\bibliography{references}
\bibliographystyle{abbrvnat}



\appendix

\clearpage
\section{Related Works}
\label{sec:related}
\paragraph{Offline reinforcement learning.}
Offline reinforcement learning studies how to learn a high return policy from a fixed dataset while mitigating distribution shift between the learned policy and the logged behavior \cite{lange2012batch,levine2020offline}. A large family of methods constrains policy improvement to remain close to actions supported by the dataset, either by explicitly regularizing the learned policy toward behavior cloning or by limiting the space of candidate actions used for value based updates \cite{wu2019brac,fujimoto2019bcq,kumar2019bear,fujimoto2021td3bc,tarasov2023rebrac}. Another family focuses on conservative value estimation to reduce overestimation when evaluating actions that are unlikely under the dataset, with CQL penalizing high values assigned to low density actions \cite{kumar2020cql}. Complementary approaches avoid explicit out of distribution action maximization during policy improvement by extracting policies from value functions using implicit or advantage weighted updates, as in IQL and AWAC \cite{kostrikov2022iql,nair2021awac}. Across these lines, the shared objective is to maximize return while ensuring that learning and improvement remain grounded in the support of the offline data.

\paragraph{Batch reinforcement learning and approximate dynamic programming.}
Offline RL is closely related to earlier work on \emph{batch} reinforcement learning and approximate dynamic programming, where policies are learned from logged transition tuples without additional interaction. Classical fitted value-iteration style methods include tree-based fitted Q-iteration \citep{ernst2005tree} and least-squares policy iteration \citep{lagoudakis2003lspi}, along with neural fitted Q-iteration (NFQ) as an early neural approach \citep{riedmiller2005nfq}. These methods clarified the basic promise of learning from static data and the practical dependence on function approximation, dataset coverage, and policy evaluation quality; for broader background, see the batch RL chapter by \citet{lange2012batch} and the modern offline RL tutorial by \citet{levine2020offline}.

\paragraph{Coverage, distribution shift, and off-policy evaluation.}
A recurring theme since early approximate dynamic programming is that learning from a fixed dataset hinges on how well the data covers the state--action regions needed by the learned policy. Analyses of fitted value iteration make this dependence explicit through approximation and sampling error terms that are amplified by distribution shift across Bellman backups \citep{munos2008fvi,antos2007fqi}. In parallel, the offline setting is tightly connected to off-policy evaluation (OPE), which estimates the performance of a target policy from data generated by different behavior policies. Doubly robust estimators \citep{jiang2016doubly} and subsequent refinements aimed at variance control and model-based mixing \citep{thomas2016dataefficient} highlight the statistical difficulty of evaluating and optimizing policies under mismatch, which later offline RL methods address implicitly through regularization or pessimism.

\paragraph{Deep model-free offline RL.}
With deep function approximation, na\"ive off-policy actor--critic can fail offline due to extrapolation error on actions that are weakly supported by the dataset, which is then compounded by bootstrapping \citep{kumar2019bear,fujimoto2019offpolicy}. Many practical algorithms can be viewed as enforcing \emph{policy support constraints} during improvement. Behavior-regularized actor--critic \citep{wu2019brac} and advantage-weighted updates \citep{peng2019awr} bias learning toward dataset actions, while critic-regularized regression \citep{wang2020crr} fits policies via weighted regression modulated by value estimates. A particularly influential baseline is TD3+BC, which augments an online actor update with a behavior cloning term and careful normalization \citep{fujimoto2021td3bc}. Another line seeks \emph{pessimism} in value estimation: Conservative Q-Learning (CQL) explicitly lowers Q-values for out-of-distribution actions \citep{kumar2020cql}, and Implicit Q-Learning (IQL) avoids explicit maximization over unseen actions by combining expectile value regression with advantage-weighted policy updates \citep{kostrikov2022iql}. Density-ratio and stationary-distribution correction viewpoints also motivate offline policy optimization; for example, \citet{lee2021optidice} estimate distribution corrections to optimize policies from fixed data.

\paragraph{Model-based offline RL.}
Model-based offline RL learns a dynamics model from logged data and uses model rollouts for planning or for generating additional training targets, but must guard against \emph{model exploitation} when rollouts drift out of support. MOPO penalizes uncertainty to discourage out-of-distribution model rollouts \citep{yu2020mopo}, MOReL constructs pessimistic MDP variants to avoid unreliable regions \citep{kidambi2020morel}, and COMBO combines model rollouts with conservative value regularization \citep{yu2021combo}. Planning-first formulations such as MBOP emphasize using learned models and short-horizon planning as the primary control mechanism under offline training \citep{argenson2021mbop}. Collectively, these works frame a core tradeoff between using model-based lookahead to improve decisions and controlling error accumulation from imperfect learned dynamics.

\paragraph{Offline RL as sequence modeling and generative policy classes.}
A complementary direction casts offline RL as conditional sequence modeling, learning to output actions from trajectory context rather than performing Bellman backups. Decision Transformer \citep{chen2021decisiontransformer} conditions on desired returns, while Trajectory Transformer treats trajectories as token sequences for planning and action selection \citep{janner2021trajectorytransformer}. More recently, \citet{emmons2021rvs} study when supervised learning alone can suffice for strong offline performance, emphasizing choices of conditioning information and model capacity. Separately, generative modeling has been used to enlarge policy classes while retaining dataset support; diffusion-based policy representations such as Diffusion-QL \citep{wang2022diffusionql} and diffusion planners \citep{janner2022diffuser} demonstrate strong performance but typically require iterative sampling at inference.

\section{Diffusion Sampler and Rollout Gradients}
\label{sec:proof_app}
\subsection{Differentiable Diffusion Sampler}
\label{sec:dds}
Our objective is to learn a parametric differentiable diffusion sampler that conditioned on a given state-action pair $s_t, a_t$ at time $t$ and sampled noise $\epsilon_{t}$ generates the next state $s_{t+1}$ as follows: 
$s_{t+1} = f_{\theta}(s_t, a_t, \varepsilon_t), \varepsilon_t \sim p_0(\varepsilon)$.
Here $f_{\theta}$ is the \emph{reverse diffusion sampler}, written as a deterministic computation graph, and $\varepsilon_t$ is a collection of Gaussian random variables used in the generation process. Equivalently, the diffusion model specifies a conditional distribution over next states, $p_{\theta}(s_{t+1}\mid s_t,a_t)$, together with a reparameterized sampling procedure $s_{t+1}=f_{\theta}(s_t,a_t,\varepsilon_t)$ whose randomness is isolated in $\varepsilon_t$.


We learn $p_{\theta}(s_{t+1}\mid s_t,a_t)$ from an offline dataset of transitions $\mathcal{D} = \{(s_t,a_t,r_t,s_{t+1})\}_{t=1}^N.$
The diffusion model is trained to represent the conditional law of $s_{t+1}$ given $(s_t,a_t)$ by introducing a \emph{forward noising process} on $s_{t+1}$ and a learned \emph{reverse denoising process} that inverts this noising when conditioned on $(s_t,a_t)$. At sampling time, the reverse process induces the map $f_{\theta}$.

\textbf{Forward process:}
Fix a diffusion horizon $K\in\mathbb{N}$ and a variance schedule $\{\alpha_k\}_{k=1}^K \subset (0,1)$. For each transition tuple $(s_t,a_t,s_{t+1})\in\mathcal{D}$, define a Markovian forward process that progressively corrupts the next state as follows:
\begin{align*}
s_{t+1}^{(0)} &= s_{t+1}, \quad
s_{t+1}^{(k)} \mid s_{t+1}^{(k-1)}
\sim q\big(s^{(k)} \mid s^{(k-1)}\big) = \mathcal{N}\big(\sqrt{\alpha_k}\, s_{t+1}^{(k-1)}, (1-\alpha_k) I\big),
 k=1,\dots,K.
\end{align*}
Let $\bar\alpha_k := \prod_{j=1}^k \alpha_j$. This choice implies a closed-form marginal for any noise level $k$ given by
$q\big(s_{t+1}^{(k)} \mid s_{t+1}\big)
=
\mathcal{N}\big(\sqrt{\bar\alpha_k}\, s_{t+1}, (1-\bar\alpha_k)I\big)$. In particular, one can write a reparameterized sample from the marginal as
\begin{align*}
s_{t+1}^{(k)} = \sqrt{\bar\alpha_k}\, s_{t+1} + \sqrt{1-\bar\alpha_k}\,\epsilon,
\qquad \epsilon \sim \mathcal{N}(0,I).
\end{align*}

\paragraph{Conditional reverse process:}
The reverse process aims to sample $s_{t+1}$ conditioned on $(s_t,a_t)$ by iteratively denoising from a Gaussian reference distribution at level $K$. Let the conditioning be $c_t := (s_t,a_t).$ At sampling time, initialize from a base noise variable $s_{t+1}^{(K)} \sim \mathcal{N}(0,I)$, and apply a learned reverse transition for $k\in[K-1]$ using $s_{t+1}^{(k-1)}
= \frac{1}{\sqrt{\alpha_k}}
  \Big(
    s_{t+1}^{(k)}
    - (1-\alpha_k)\, \hat\epsilon_{\theta}\big(s_{t+1}^{(k)}, k, c_t\big)
  \Big)
  + \sigma_k z_{k-1}, 
z_{k-1} \sim \mathcal{N}(0,I).$
Here $\hat\epsilon_{\theta}(\cdot,k,c_t)$ is a parametric predictor of the noise component at level $k$, and $\{\sigma_k\}$ specifies the reverse-process variance. The final denoised sample is $s_{t+1}^{(0)} \sim p_{\theta}(\cdot \mid s_t,a_t)$, and we define the reverse sampler $f_{\theta}$ by collecting all Gaussian random variables used by the reverse procedure into $\varepsilon_t := (z_K,z_{K-1},\dots,z_0)$, so that the sampled next state can be written as $s_{t+1}^{(0)} = f_{\theta}(s_t,a_t,\varepsilon_t)$.
We define the initialization map $g_K:\mathbb{R}^d\to\mathbb{R}^d$ and the reverse-step map
$h_k:\mathbb{R}^d\times\mathbb{R}^d\times\mathbb{R}^m\times\mathbb{R}^d\to\mathbb{R}^d$ by
\begin{align}
g_K(z) &:= z,\nonumber\\
h_k(u,s,a,z)
&:=
\frac{1}{\sqrt{\alpha_k}}
\Big(
u-(1-\alpha_k)\,
\hat\epsilon_{\theta}\big(u,k,(s,a)\big)
\Big)+\sigma_k z,\;\;\;\;\;
\text{for } k=1,\ldots,K-1
\label{eq:g_h_def}
\end{align}
so that $s_{t+1}^{(k-1)}\!\!=\!h_k(s_{t+1}^{(k)},s_t,a_t,z_{k-1})$ and $s_{t+1}^{(K)}=g_K(z_K)$.

\paragraph{Deterministic computation graph for fixed noise.}
For any fixed realization of $\varepsilon_t$, the mapping $(s_t,a_t)\mapsto s_{t+1}^{(0)}$ is a deterministic composition of (i) linear operations, (ii) evaluations of the denoiser $\hat\epsilon_{\theta}(\cdot,k,c_t)$ at each reverse step, and (iii) additive terms determined by the fixed Gaussian draws. Consequently, $f_{\theta}$ is a differentiable computation graph in its inputs $(s_t,a_t)$ for fixed $\varepsilon_t$.

\paragraph{Learning objective:}
We train $\hat\epsilon_{\theta}$ to invert the forward corruption of $s_{t+1}$, conditioned on $(s_t,a_t)$. Using the marginal reparameterization at a randomly chosen diffusion level $k$, we form
$s_{t+1}^{(k)} = \sqrt{\bar\alpha_k}\, s_{t+1} + \sqrt{1-\bar\alpha_k}\,\epsilon,
\qquad \epsilon \sim \mathcal{N}(0,I)$,
and minimize the conditional noise-prediction error
\begin{align*}
\mathcal{L}(\theta)
&=
\mathbb{E}_{\substack{(s_t,a_t,s_{t+1})\sim \mathcal{D}\\ k \sim \mathrm{Unif}(\{1,\dots,K\})\\ \epsilon \sim \mathcal{N}(0,I)}}
\Big[
\big\|
\epsilon - \hat\epsilon_{\theta}\big(s_{t+1}^{(k)}, k, c_t\big)
\big\|_2^2
\Big].
\end{align*}
Intuitively, this objective teaches the denoiser to recover the injected Gaussian noise at arbitrary noise levels while leveraging $(s_t,a_t)$ as side information. After training, the resulting reverse process defines a conditional generative model $p_{\theta}(s_{t+1}\mid s_t,a_t)$, and its sampling procedure is precisely the transition map
$s_{t+1} = f_{\theta}(s_t,a_t,\varepsilon_t), \varepsilon_t \sim p_0(\varepsilon)$,
which we treat as a learned, differentiable simulator of one-step dynamics.

\subsection{Gradient recursion through diffusion rollouts}
\label{sec:mpc_backprop_diffusion_critic}
Recall that given a fixed noise sequence
$\varepsilon_{t:t+H-1}=(\varepsilon_t,\ldots,\varepsilon_{t+H-1})$
define the imagined rollout recursively by
\begin{align}
\tilde{s}_0 &= s_t,
\,\,
\tilde{a}_j = \pi_{\psi}(\tilde{s}_j),
\,\,
\tilde{s}_{j+1} = f_{\theta}(\tilde{s}_j,\tilde{a}_j,\varepsilon_{t+j}),
\label{eq:mpc_rollout_recursion}
\end{align}
The next theorem states a compact Jacobian recursion for the gradient of the per-noise objective
$L(\psi;\varepsilon_{t:t+H-1})$ under the rollout dynamics in
\eqref{eq:mpc_rollout_recursion}. Define the rollout Jacobians, for $j=0,\ldots,H-1$, 
\begin{align*}
\Pi_s(j) &:= \nabla_s \pi_{\psi}(s)\big|_{s=\tilde{s}_j},
\quad
\Pi_{\psi}(j) := \nabla_{\psi}\pi_{\psi}(\tilde{s}_j),
\\
F_s(j) &:= \nabla_s f_{\theta}(s,a,\varepsilon_{t+j})\big|_{s=\tilde{s}_j,a=\tilde{a}_j},
\\
F_a(j) &:= \nabla_a f_{\theta}(s,a,\varepsilon_{t+j})\big|_{s=\tilde{s}_j,a=\tilde{a}_j}.
\end{align*}
Let $G_j := \nabla_{\psi}\tilde{s}_j$ and $D_j := \nabla_{\psi}\tilde{a}_j$. Then $G_0=0$ and, for $j=0,\ldots,H-1$,
\begin{align*}
D_j &= \Pi_s(j)\,G_j + \Pi_{\psi}(j),
\\
G_{j+1} &= F_s(j)\,G_j + F_a(j)\,D_j.
\end{align*}
Further define the gradient of the reward and the critic as:
\begin{align*}
r_s(j) := \nabla_s r_{\xi}(\tilde{s}_j,\tilde{a}_j),
\;\;
r_a(j) := \nabla_a r_{\xi}(\tilde{s}_j,\tilde{a}_j),\\
Q_s := \nabla_s Q_{\phi}(\tilde{s}_H,\tilde{a}_H),
\;\;
Q_a := \nabla_a Q_{\phi}(\tilde{s}_H,\tilde{a}_H).
\end{align*}
\begin{tcolorbox}[
  width=\textwidth,
  colback=gray!1,        
  colframe=black,         
  arc=3pt,               
  boxrule=1pt,            
  left=3pt, right=3pt,  
  top=3pt,  bottom=2pt  
]
\begin{theorem}[\textbf{Gradient recursion}]
\label{thm:grad_recursion_diffusion_mpc}
Fix a time $t$, a horizon $H$, and a noise sequence $\varepsilon_{t:t+H-1}$, and let
$\{(\tilde{s}_j,\tilde{a}_j)\}_{j=0}^{H}$ be defined by \eqref{eq:mpc_rollout_recursion}.
Assume $\pi_{\psi}$ is differentiable in $\psi$ and its state input, and $f_{\theta}$, $r_{\xi}$, and $Q_{\phi}$
are differentiable in their state and action arguments.

Then the gradient of the per-noise return is
\begin{align*}
\nabla_{\psi} L(\psi;\varepsilon_{t:t+H-1})
&=
\sum_{j=0}^{H-1}\gamma^j\,
\Big(
r_s(j)\,G_j + r_a(j)\,D_j
\Big) \nonumber
\\
&+
\gamma^H\,
\Big(
Q_s\,G_H + Q_a\,D_H
\Big).
\end{align*}
Moreover, if $f_{\theta}$ is implemented by a reverse diffusion recursion as given by \eqref{eq:g_h_def}
then
\begin{align*}
\nabla_a f_{\theta}(s,a,\varepsilon) = A_0,
\qquad
\nabla_s f_{\theta}(s,a,\varepsilon) = B_0,
\end{align*}
where $A_{K}=0$, $B_{K}=0$, and for $k=K,\ldots,1$,
\begin{align*}
A_{k-1}
&=
\frac{\partial h_{k}}{\partial u}\,A_{k}
+
\frac{\partial h_{k}}{\partial a},
\qquad
B_{k-1}
=
\frac{\partial h_{k}}{\partial u}\,B_{k}
+
\frac{\partial h_{k}}{\partial s},
\end{align*}
with all partial derivatives evaluated at $(u,s,a,z)=(s^{(k)},s,a,z_{k-1})$.
\end{theorem}
\end{tcolorbox}
\vspace{1.5ex}

\begin{proof}
    
We fix a time $t$, horizon $H$, and a noise sequence $\varepsilon_{t:t+H-1}$.
All quantities below are defined along the imagined rollout \eqref{eq:mpc_rollout_recursion}.
Since the noise is fixed, the rollout is deterministic, hence the gradient
$\nabla_{\psi}L(\psi;\varepsilon_{t:t+H-1})$ is obtained by repeated application of the chain rule.

\begin{enumerate}
    \item \textbf{Sensitivity recursions:} Recall the rollout recursion \eqref{eq:mpc_rollout_recursion}:
    \begin{align}
    \tilde{s}_0 &= s_t,
    \qquad
    \tilde{a}_h = \pi_{\psi}(\tilde{s}_h),
    \qquad
    \tilde{s}_{h+1} = f_{\theta}(\tilde{s}_h,\tilde{a}_h,\varepsilon_{t+h}),
    \qquad h=0,\ldots,H-1.
    \label{eq:app_rollout_recursion}
    \end{align}
    Define the state and action sensitivities
    \begin{align}
    G_h := \nabla_{\psi}\tilde{s}_h,
    \qquad
    D_h := \nabla_{\psi}\tilde{a}_h.
    \label{eq:app_def_G_D}
    \end{align}
    Since $\tilde{s}_0=s_t$ does not depend on $\psi$, we have
    \begin{align}
    G_0 = 0.
    \label{eq:app_G0}
    \end{align}
    
    For each $h$, the action is $\tilde{a}_h=\pi_{\psi}(\tilde{s}_h)$. By the chain rule,
    \begin{align}
    D_h
    =
    \nabla_{\psi}\pi_{\psi}(\tilde{s}_h)
    +
    \nabla_s \pi_{\psi}(s)\big|_{s=\tilde{s}_h}\,\nabla_{\psi}\tilde{s}_h.
    \label{eq:app_chain_policy}
    \end{align}
    Introduce the policy Jacobians (as in the theorem statement)
    \begin{align}
    \Pi_s(h) := \nabla_s \pi_{\psi}(s)\big|_{s=\tilde{s}_h},
    \qquad
    \Pi_{\psi}(h) := \nabla_{\psi}\pi_{\psi}(\tilde{s}_h).
    \label{eq:app_def_Pi}
    \end{align}
    Substituting \eqref{eq:app_def_Pi} into \eqref{eq:app_chain_policy} yields
    \begin{align}
    D_h = \Pi_s(h)\,G_h + \Pi_{\psi}(h).
    \label{eq:app_D_recursion}
    \end{align}
    
    For each $h\in\{0,\ldots,H-1\}$, the next state is
    $\tilde{s}_{h+1}=f_{\theta}(\tilde{s}_h,\tilde{a}_h,\varepsilon_{t+h})$.
    Differentiating with respect to $\psi$ and applying the chain rule gives
    \begin{align}
    G_{h+1}
    =
    \nabla_s f_{\theta}(s,a,\varepsilon_{t+h})\big|_{s=\tilde{s}_h,a=\tilde{a}_h}\,G_h
    +
    \nabla_a f_{\theta}(s,a,\varepsilon_{t+h})\big|_{s=\tilde{s}_h,a=\tilde{a}_h}\,D_h.
    \label{eq:app_chain_dynamics}
    \end{align}
    Introduce the world-model Jacobians (as in the theorem statement)
    \begin{align}
    F_s(h) := \nabla_s f_{\theta}(s,a,\varepsilon_{t+h})\big|_{s=\tilde{s}_h,a=\tilde{a}_h},
    \qquad
    F_a(h) := \nabla_a f_{\theta}(s,a,\varepsilon_{t+h})\big|_{s=\tilde{s}_h,a=\tilde{a}_h}.
    \label{eq:app_def_F}
    \end{align}
    Substituting \eqref{eq:app_def_F} into \eqref{eq:app_chain_dynamics} yields
    \begin{align}
    G_{h+1} = F_s(h)\,G_h + F_a(h)\,D_h.
    \label{eq:app_G_recursion}
    \end{align}
    
    Equations \eqref{eq:app_G0}, \eqref{eq:app_D_recursion}, and \eqref{eq:app_G_recursion} are exactly the sensitivity recursions
    stated in Theorem~\ref{thm:grad_recursion_diffusion_mpc}.

    \item \textbf{Gradient of the per-noise return}
    
    Recall the per-noise return definition:
    \begin{align}
    L(\psi;\varepsilon_{t:t+H-1})
    =
    \sum_{h=0}^{H-1}\gamma^h\, r_{\xi}(\tilde{s}_h,\tilde{a}_h)
    +
    \gamma^H\, Q_{\phi}\!\Big(\tilde{s}_H,\pi_{\psi}(\tilde{s}_H)\Big).
    \label{eq:app_def_L}
    \end{align}
    
    Fix $h\in\{0,\ldots,H-1\}$ and define the stage reward
    $r_h := r_{\xi}(\tilde{s}_h,\tilde{a}_h)$.
    By the chain rule,
    \begin{align}
    \nabla_{\psi} r_h
    =
    \nabla_s r_{\xi}(\tilde{s}_h,\tilde{a}_h)\,G_h
    +
    \nabla_a r_{\xi}(\tilde{s}_h,\tilde{a}_h)\,D_h.
    \label{eq:app_grad_stage_reward}
    \end{align}
    Introduce the shorthand
    \begin{align}
    r_s(h) := \nabla_s r_{\xi}(\tilde{s}_h,\tilde{a}_h),
    \qquad
    r_a(h) := \nabla_a r_{\xi}(\tilde{s}_h,\tilde{a}_h).
    \label{eq:app_def_rs_ra}
    \end{align}
    Then \eqref{eq:app_grad_stage_reward} becomes
    \begin{align}
    \nabla_{\psi} r_{\xi}(\tilde{s}_h,\tilde{a}_h)
    =
    r_s(h)\,G_h + r_a(h)\,D_h.
    \label{eq:app_grad_stage_compact}
    \end{align}
    
    Define the terminal action $\tilde{a}_H := \pi_{\psi}(\tilde{s}_H)$ and the terminal value
    $V_T := Q_{\phi}(\tilde{s}_H,\tilde{a}_H)$.
    By the chain rule,
    \begin{align}
    \nabla_{\psi} V_T
    =
    \nabla_s Q_{\phi}(\tilde{s}_H,\tilde{a}_H)\,G_H
    +
    \nabla_a Q_{\phi}(\tilde{s}_H,\tilde{a}_H)\,D_H.
    \label{eq:app_grad_terminal_value}
    \end{align}
    Introduce the terminal derivatives
    \begin{align}
    Q_s := \nabla_s Q_{\phi}(\tilde{s}_H,\tilde{a}_H),
    \qquad
    Q_a := \nabla_a Q_{\phi}(\tilde{s}_H,\tilde{a}_H).
    \label{eq:app_def_Qs_Qa}
    \end{align}
    Then \eqref{eq:app_grad_terminal_value} becomes
    \begin{align}
    \nabla_{\psi} Q_{\phi}(\tilde{s}_H,\tilde{a}_H)
    =
    Q_s\,G_H + Q_a\,D_H.
    \label{eq:app_grad_terminal_compact}
    \end{align}
    Here $D_H$ is defined exactly as in \eqref{eq:app_def_G_D}, and can be expanded via the same policy sensitivity identity
    \eqref{eq:app_D_recursion} with $h=H$:
    \begin{align}
    D_H = \Pi_s(H)\,G_H + \Pi_{\psi}(H).
    \label{eq:app_DH}
    \end{align}
    
    Differentiating \eqref{eq:app_def_L} and applying \eqref{eq:app_grad_stage_compact} and \eqref{eq:app_grad_terminal_compact} gives
    \begin{align}
    \nabla_{\psi} L(\psi;\varepsilon_{t:t+H-1})
    &=
    \sum_{h=0}^{H-1}\gamma^h\,
    \nabla_{\psi} r_{\xi}(\tilde{s}_h,\tilde{a}_h)
    +
    \gamma^H\,
    \nabla_{\psi} Q_{\phi}(\tilde{s}_H,\tilde{a}_H)
    \nonumber\\
    &=
    \sum_{h=0}^{H-1}\gamma^h\,
    \Big(
    r_s(h)\,G_h + r_a(h)\,D_h
    \Big)
    +
    \gamma^H\,
    \Big(
    Q_s\,G_H + Q_a\,D_H
    \Big),
    \label{eq:app_grad_L_final}
    \end{align}
    which is the gradient expression stated in Theorem~\ref{thm:grad_recursion_diffusion_mpc}.

    \item \textbf{Jacobian Recursion:} Next we derive and prove the Jacobian recursions for $\nabla_a f_{\theta}(s,a,\varepsilon)$ and $\nabla_s f_{\theta}(s,a,\varepsilon)$ under the reverse diffusion implementation \eqref{eq:g_h_def}.
    
    Fix $(s,a)$ and a noise pack $\varepsilon=(z_K,z_{K-1},\ldots,z_0)$. Let $\{s^{(k)}\}_{k=0}^{K}$ be generated by
    \begin{align}
    s^{(K)} &= g_K(z_K), \label{eq:app_depth_init}\\
    s^{(k-1)} &= h_k\big(s^{(k)},s,a,z_{k-1}\big),
    \qquad k=K,\ldots,1, \label{eq:app_depth_step}\\
    f_{\theta}(s,a,\varepsilon) &= s^{(0)}. \label{eq:app_depth_out}
    \end{align}
    Assume $g_K$ is independent of $(s,a)$ and each $h_k(u,s,a,z)$ is differentiable in $(u,s,a)$.
    
    Define the diffusion-depth sensitivities
    \begin{align}
    A_k := \nabla_a s^{(k)} \in \mathbb{R}^{d\times m},
    \qquad
    B_k := \nabla_s s^{(k)} \in \mathbb{R}^{d\times d}.
    \label{eq:app_depth_def_AB}
    \end{align}
    
    Since $s^{(K)}=g_K(z_K)$ and $g_K$ does not depend on $(s,a)$,
    \begin{align}
    A_K = 0,
    \qquad
    B_K = 0.
    \label{eq:app_depth_init_AB}
    \end{align}
    
    Fix $k\in\{1,\ldots,K\}$ and write \eqref{eq:app_depth_step} as
    \begin{align}
    s^{(k-1)} = h_k(u,s,a,z_{k-1})\Big|_{u=s^{(k)}}.
    \label{eq:app_depth_step_u}
    \end{align}
    Differentiate \eqref{eq:app_depth_step_u} with respect to $a$ and apply the chain rule:
    \begin{align}
    A_{k-1}
    &=
    \nabla_a s^{(k-1)}
    =
    \frac{\partial h_k}{\partial u}\,\nabla_a s^{(k)}
    +
    \frac{\partial h_k}{\partial a}
    =
    \frac{\partial h_k}{\partial u}\,A_k
    +
    \frac{\partial h_k}{\partial a},
    \label{eq:app_depth_A_rec}
    \end{align}
    where the partial derivatives are evaluated at
    \begin{align}
    (u,s,a,z)=\big(s^{(k)},\,s,\,a,\,z_{k-1}\big).
    \label{eq:app_depth_eval_point}
    \end{align}
    
    Differentiating \eqref{eq:app_depth_step_u} with respect to $s$ gives
    \begin{align}
    B_{k-1}
    &=
    \nabla_s s^{(k-1)}
    =
    \frac{\partial h_k}{\partial u}\,\nabla_s s^{(k)}
    +
    \frac{\partial h_k}{\partial s}
    =
    \frac{\partial h_k}{\partial u}\,B_k
    +
    \frac{\partial h_k}{\partial s},
    \label{eq:app_depth_B_rec}
    \end{align}
    with the same evaluation point \eqref{eq:app_depth_eval_point}.
    
    \paragraph{Iterating to obtain $\nabla_a f_{\theta}$ and $\nabla_s f_{\theta}$.}
    Starting from \eqref{eq:app_depth_init_AB} and iterating \eqref{eq:app_depth_A_rec}--\eqref{eq:app_depth_B_rec} for
    $k=K,K-1,\ldots,1$ yields $A_0$ and $B_0$. Since $f_{\theta}(s,a,\varepsilon)=s^{(0)}$ by \eqref{eq:app_depth_out},
    \begin{align}
    \nabla_a f_{\theta}(s,a,\varepsilon) = \nabla_a s^{(0)} = A_0,
    \qquad
    \nabla_s f_{\theta}(s,a,\varepsilon) = \nabla_s s^{(0)} = B_0.
    \label{eq:app_depth_conclusion}
    \end{align}
\end{enumerate}

\end{proof}

\section{Additional Experimental Details}
\label{sec:exp_details}

\subsection{World Model}
\label{sec:add_world_model}
\textbf{Diffusion Sampler:} We train a conditional DDPM dynamics model to predict \(s_{t+1}\) from \((s_t,a_t)\) on offline D4RL MuJoCo datasets (Hopper, HalfCheetah, and Walker2d, using standard v2 splits). States and next states are standardized with dataset mean/std (with a small \(\epsilon\) for numerical stability), while actions are left unnormalized, and training minimizes the \(\epsilon\)-prediction mean-squared error objective. The denoiser is an MLP with time-embedding dimension \(128\), hidden size \(512\), and \(4\) layers, using \(T=200\) diffusion steps and a linear \(\beta\)-schedule from \(10^{-4}\) to \(2\times10^{-2}\). Optimization uses AdamW (learning rate \(3\times10^{-4}\), weight decay \(0\)), batch size \(2048\), gradient clipping at \(1.0\), exponential moving average (EMA) of parameters with decay \(0.999\) and is trained for \(200\mathrm{k}\) steps.

\textbf{Reward Model:} We train a reward model \(r_\theta(s,a)\) as an MLP on D4RL transitions, where the input is the concatenated state-action vector \([s,a]\) and the output is a scalar immediate reward prediction. The network uses hidden width \(512\) with \(4\) layers total (\(3\) hidden layers plus a final linear output layer). States are normalized using dataset statistics while actions are left unnormalized, and training minimizes MSE with AdamW (learning rate \(3\times10^{-4}\), batch size \(2048\)) for \(200\mathrm{k}\) steps with gradient clipping at \(1.0\). We checkpoint every \(20\mathrm{k}\) steps and run periodic online validation every \(5\mathrm{k}\) steps.

We report the predictive performance of the offline trained diffusion dynamics model and the offline trained reward model across D4RL MuJoCo datasets. Figure~\ref{fig:rmse_diffusion} shows the diffusion model prediction RMSE across training steps, and Figure~\ref{fig:rmse_reward} reports the reward model prediction RMSE across training steps. 
For final performance, Table~\ref{tab:diffusion_reward_rmse_d4rl} reports both diffusion model accuracy, measured by one-step RMSE, and reward model accuracy, measured by reward prediction RMSE, with uncertainty reported as standard error. For the diffusion model, the one-step RMSE is computed using Monte Carlo samples and normalized by the number of transitions. Overall, both models achieve low error on most datasets.

\begin{figure*}[t]
        \includegraphics[width=\textwidth]{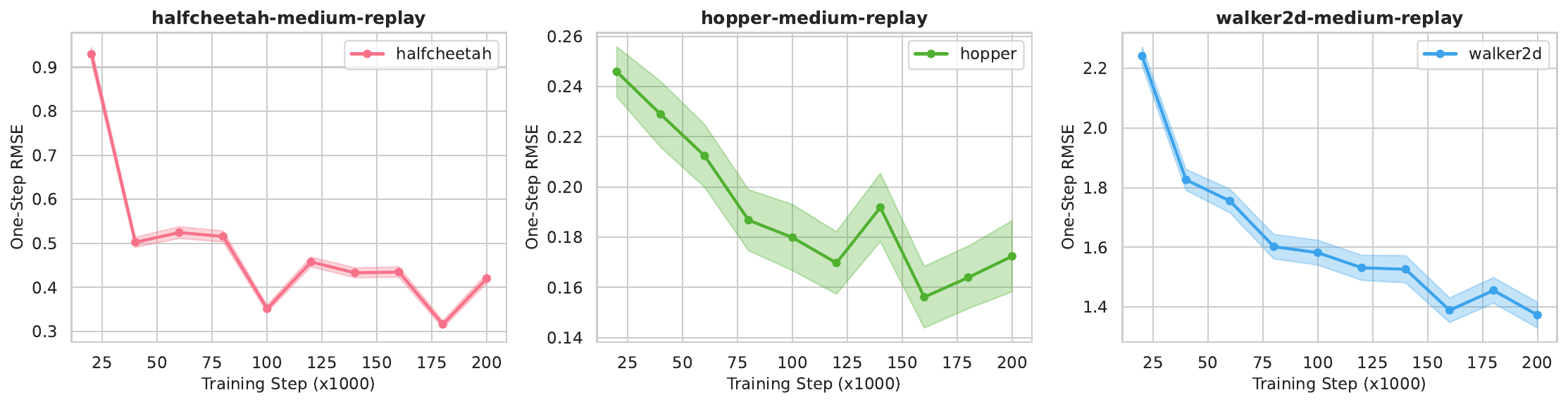}
    \caption{One-step state prediction RMSE of diffusion models across training steps (20k–200k) on medium-replay datasets. RMSE decreases with training for all three environments (halfcheetah, hopper, walker2d), with shaded regions showing standard error over 1000 transitions.} 
    \label{fig:rmse_diffusion}
\end{figure*}

\begin{figure*}[t]
        \includegraphics[width=\textwidth]{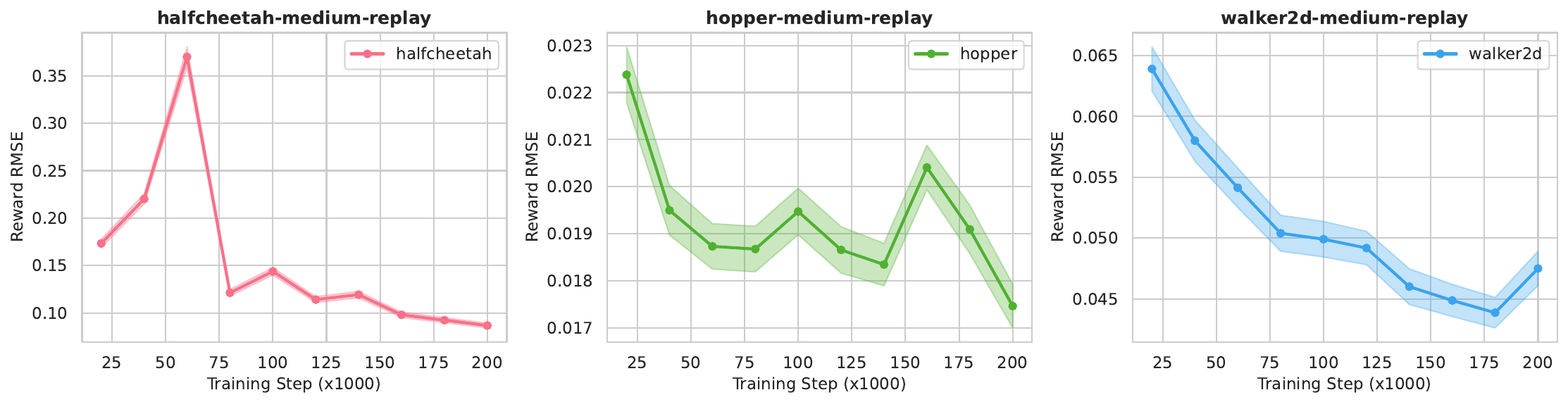}
    \caption{Reward prediction RMSE of reward models across training steps (20k–200k) on medium-replay datasets. All environments show decreasing RMSE with training. Shaded regions indicate standard error over 1000 transitions.} 
    \label{fig:rmse_reward}
\end{figure*}

\begin{table*}[h]
\centering
\begin{tabular}{l c c}
\toprule
Task & Diffusion one-step RMSE ($\pm$ SE) & Reward RMSE ($\pm$ SE) \\
\midrule
halfcheetah-random & 0.2318 $\pm$ 0.000010 & 0.0477 $\pm$ 0.0016 \\
halfcheetah-medium & 0.5494 $\pm$ 0.000014 & 0.1292 $\pm$ 0.0034 \\
halfcheetah-expert & 1.3065 $\pm$ 0.000026 & 0.2763 $\pm$ 0.0076 \\
halfcheetah-medium-expert & 0.5972 $\pm$ 0.000013 & 0.1622 $\pm$ 0.0045 \\
halfcheetah-medium-replay & 0.3996 $\pm$ 0.000014 & 0.0918 $\pm$ 0.0027 \\
halfcheetah-full-replay & 0.4739 $\pm$ 0.000014 & 0.1164 $\pm$ 0.0033 \\
\midrule
hopper-random & 0.1268 $\pm$ 0.004497 & 0.0082 $\pm$ 0.0003 \\
hopper-medium & 0.0784 $\pm$ 0.001281 & 0.0720 $\pm$ 0.0030 \\
hopper-expert & 0.1375 $\pm$ 0.000755 & 0.1415 $\pm$ 0.0058 \\
hopper-medium-expert & 0.1678 $\pm$ 0.000710 & 0.0789 $\pm$ 0.0033 \\
hopper-medium-replay & 0.1008 $\pm$ 0.000764 & 0.0192 $\pm$ 0.0005 \\
hopper-full-replay & 0.1178 $\pm$ 0.002514 & 0.0208 $\pm$ 0.0005 \\
\midrule
walker2d-random & 0.6591 $\pm$ 0.009463 & 0.0244 $\pm$ 0.0008 \\
walker2d-medium & 1.4068 $\pm$ 0.005905 & 0.0779 $\pm$ 0.0020 \\
walker2d-expert & 3.3412 $\pm$ 0.014321 & 0.4902 $\pm$ 0.0123 \\
walker2d-medium-expert & 1.5566 $\pm$ 0.026789 & 0.1060 $\pm$ 0.0031 \\
walker2d-medium-replay & 1.1100 $\pm$ 0.002611 & 0.0459 $\pm$ 0.0013 \\
walker2d-full-replay & 1.2560 $\pm$ 0.009531 & 0.0578 $\pm$ 0.0022 \\
\midrule
Average & 0.7565 $\pm$ 0.000031 & 0.1092 $\pm$ 0.0271 \\
\bottomrule
\end{tabular}
\caption{Diffusion one-step RMSE and reward prediction RMSE (mean $\pm$ standard error) on D4RL MuJoCo datasets.}
\label{tab:diffusion_reward_rmse_d4rl}
\end{table*}

\textbf{Policy and Critic Value:} ReBRAC is a TD3+BC-style offline RL algorithm that uses three-hidden-layer actor and critic networks with hidden dimension $256$ and ReLU activations. The critic applies LayerNorm between hidden layers, and both the actor and critic are trained with behavior-regularization penalties whose coefficients are tuned separately. For D4RL Gym-MuJoCo tasks, ReBRAC uses Adam with batch size $1024$, learning rate $10^{-3}$ for all networks, target update rate $\tau=5\times 10^{-3}$, and discount factor $\gamma=0.99$. The main tuned hyperparameters are the actor and critic regularization coefficients $\beta_1$ and $\beta_2$. We follow the hyperparameter choices from \citet{tarasov2023minimalist}.

\subsection{Inference time Parameters}
The inference time parameters for our method are: Horizon Length ($H$), Inner Gradient Steps ($E$), Step size ($\eta$), and number of particles $M$. We tune them on a single different seed, which is different from the 10 evaluation seeds. We tune the planning horizon $H \in \{0,1,2,5,10\}$, the number of inner gradient steps $E \in \{1,2\}$, and the number of particles $M \in \{2048,4096,8192\}$. For HalfCheetah and Walker2d, we tune the inner step size $\eta \in \{5 \times 10^{-7}, 5 \times 10^{-6}, 5 \times 10^{-5}, 5 \times 10^{-4}\}$ and for Hopper, $\eta \in \{5 \times 10^{-8}, 5 \times 10^{-7}, 5 \times 10^{-6}, 1 \times 10^{-5}, 1 \times 10^{-4}\}$.
We keep the DDIM denoising steps $K = 10$. The choices for each dataset is given below.

\begin{table}[h]
\centering
\caption{Hyperparameters for DiffMPC across D4RL datasets. $H$: horizon length, $E$: inner gradient steps, $\eta$: step size, $M$: number of particles.}
\label{tab:hyperparameters}
\begin{tabular}{l|c|c|c|c}
\toprule
Dataset & $H$ & $E$ & $\eta$ & $M$ \\
\midrule
\multicolumn{5}{c}{\textit{halfcheetah}} \\
\midrule
medium         & 5 & 1 & $5 \times 10^{-4}$ & 8{,}192 \\
expert         & 2 & 1 & $5 \times 10^{-6}$ & 4{,}096 \\
medium-expert  & 5 & 1 & $5 \times 10^{-5}$ & 8{,}192 \\
medium-replay  & 5 & 1 & $5 \times 10^{-4}$ & 8{,}192 \\
full-replay    & 5 & 1 & $5 \times 10^{-4}$ & 4{,}096 \\
\midrule
\multicolumn{5}{c}{\textit{hopper}} \\
\midrule
medium         & 2 & 1 & $5 \times 10^{-5}$ & 8{,}192 \\
expert         & 1 & 1 & $5 \times 10^{-8}$ & 8{,}192 \\
medium-expert  & 2 & 1 & $1 \times 10^{-6}$ & 8{,}192 \\
medium-replay  & 5 & 1 & $1 \times 10^{-4}$ & 4{,}096 \\
full-replay    & 5 & 1 & $1 \times 10^{-4}$ & 4{,}096 \\
\midrule
\multicolumn{5}{c}{\textit{walker2d}} \\
\midrule
medium         & 5  & 1 & $5 \times 10^{-5}$ & 4{,}096 \\
expert         & 5  & 1 & $5 \times 10^{-5}$ & 8{,}192 \\
medium-expert  & 10 & 1 & $5 \times 10^{-5}$ & 8{,}192 \\
medium-replay  & 5  & 1 & $5 \times 10^{-5}$ & 8{,}192 \\
full-replay    & 5  & 1 & $5 \times 10^{-5}$ & 8{,}192 \\
\bottomrule
\end{tabular}
\end{table}

\section{Ablation Studies}
\label{sec:ablations}
\subsection{Horizon Length}

We conduct ablation studies to understand the effect of the planning horizon $H$ on the performance of $\DiffMPC$. Recall that $H$ controls the number of steps over which imagined rollouts are unrolled before the terminal critic $Q_{\phi}$ is applied.

\textbf{Effect of Planning Horizon $H$.} We sweep $H \in \{1, 2, 5, 10, 20\}$ across all 18 Gym-MuJoCo datasets and report normalized scores in Tables~\ref{tab:h_ablation_halfcheetah}--\ref{tab:h_ablation_walker2d}.

Several consistent patterns emerge across environments and datasets. First, even a short horizon of $H=1$ or $H=2$ is sometimes sufficient to recover the gains from inference-time adaptation, suggesting that the immediate one-step gradient signal through the world model provides useful policy improvement. This is particularly visible in the Walker2d family of tasks, where normalized scores improve monotonically or near-monotonically with $H$ (e.g., walker2d-expert improves from $112.94$ at $H=1$ to $113.74$ at $H=20$, and walker2d-medium-expert from $111.71$ to $112.77$).

Second, moderate horizons ($H=5$ or $H=10$) tend to perform well across most datasets without incurring instability. Longer horizons ($H=20$) occasionally exhibit higher variance or slight degradation, likely due to compounding model errors over longer imagined rollouts. This is most apparent in the hopper-expert task, where the standard deviation grows substantially at $H=20$ ($94.03 \pm 18.46$) compared to shorter horizons.

Third, on tasks where the pre-trained ReBRAC policy is already near-optimal or the dataset is uninformative (e.g., random datasets such as halfcheetah-random, hopper-random, walker2d-random), inference-time adaptation provides little to no gain across all horizons, consistent with our main results where we observed no improvement on these tasks.

Fourth, on full-replay and medium-replay datasets, where the offline data covers a wider range of behaviors, the gains from increasing $H$ tend to be more stable, as the world model trained on richer data yields more reliable multi-step predictions.

Overall, the ablation confirms that inference-time adaptation is robust across a broad range of horizon choices, with $H \in \{5, 10\}$ offering a good balance between performance and computational cost. Since each additional horizon step requires an extra forward pass through the diffusion sampler and reward model, as well as additional backpropagation through the computation graph, practitioners operating under strict latency constraints may prefer smaller $H$ with minimal loss.

\begin{table}[!htbp]
\centering
\small
\setlength{\tabcolsep}{6pt}
\renewcommand{\arraystretch}{1.05}
\resizebox{\linewidth}{!}{%
\begin{tabular}{l|ccccc}
\toprule
Dataset & $H=1$ & $H=2$ & $H=5$ & $H=10$ & $H=20$ \\
\midrule
random        & $33.66 \pm 0.19$  & $33.60 \pm 0.32$  & $33.76 \pm 0.19$  & $\mathbf{34.39} \pm 0.23$  & $34.23 \pm 0.34$ \\
medium        & $68.27 \pm 0.58$  & $65.52 \pm 8.85$  & $\mathbf{68.97} \pm 0.68$  & $68.88 \pm 0.66$  & $68.77 \pm 0.73$ \\
medium-replay & $51.13 \pm 0.85$  & $51.03 \pm 0.56$  & $\mathbf{51.68} \pm 0.66$  & $51.43 \pm 0.94$  & $51.48 \pm 0.67$ \\
medium-expert & $\mathbf{106.68} \pm 0.84$ & $106.13 \pm 1.31$ & $102.36 \pm 10.22$ & $105.79 \pm 0.62$ & $106.05 \pm 0.43$ \\
expert        & $105.54 \pm 0.53$ & $104.87 \pm 1.84$ & $106.06 \pm 0.36$ & $105.80 \pm 0.46$ & $\mathbf{106.67} \pm 0.36$ \\
full-replay   & $83.78 \pm 0.36$  & $84.40 \pm 0.30$  & $\mathbf{85.15} \pm 0.41$  & $84.27 \pm 0.49$  & $82.51 \pm 0.48$ \\
\bottomrule
\end{tabular}%
}
\vspace{1ex}
\caption{Effect of planning horizon $H$ on $\DiffMPC$ for HalfCheetah datasets. Normalized scores (mean $\pm$ std) on D4RL. Best score per row is in bold.}
\label{tab:h_ablation_halfcheetah}
\end{table}

\begin{table}[!htbp]
\centering
\small
\setlength{\tabcolsep}{6pt}
\renewcommand{\arraystretch}{1.05}
\resizebox{\linewidth}{!}{%
\begin{tabular}{l|ccccc}
\toprule
Dataset & $H=1$ & $H=2$ & $H=5$ & $H=10$ & $H=20$ \\
\midrule
random        & $9.01 \pm 1.14$    & $8.87 \pm 1.27$    & $8.99 \pm 1.12$   & $9.04 \pm 1.22$   & $\mathbf{9.48} \pm 1.22$ \\
medium-replay & $101.68 \pm 0.44$  & $101.83 \pm 0.31$  & $101.68 \pm 0.41$ & $93.44 \pm 24.63$ & $\mathbf{101.96} \pm 0.42$ \\
medium-expert & $112.47 \pm 1.25$  & $106.48 \pm 13.37$ & $112.99 \pm 1.07$ & $\mathbf{113.54} \pm 1.14$ & $109.38 \pm 12.93$ \\
full-replay   & $108.01 \pm 0.09$  & $108.02 \pm 0.13$  & $108.06 \pm 0.12$ & $108.38 \pm 0.10$ & $\mathbf{108.61} \pm 0.09$ \\
expert        & $107.70 \pm 13.12$ & $\mathbf{110.36} \pm 9.51$ & $101.86 \pm 17.29$ & $98.34 \pm 19.36$ & $94.03 \pm 18.46$ \\
\bottomrule
\end{tabular}%
}
\vspace{1ex}
\caption{Effect of planning horizon $H$ on $\DiffMPC$ for Hopper datasets. Normalized scores (mean $\pm$ std) on D4RL. Best score per row is in bold.}
\label{tab:h_ablation_hopper}
\end{table}

\begin{table}[!htbp]
\centering
\small
\setlength{\tabcolsep}{6pt}
\renewcommand{\arraystretch}{1.05}
\resizebox{\linewidth}{!}{%
\begin{tabular}{l|ccccc}
\toprule
Dataset & $H=1$ & $H=2$ & $H=5$ & $H=10$ & $H=20$ \\
\midrule
random        & $\mathbf{21.96} \pm 0.01$ & $\mathbf{21.96} \pm 0.02$ & $\mathbf{21.96} \pm 0.01$ & $\mathbf{21.96} \pm 0.02$ & $\mathbf{21.96} \pm 0.02$ \\
medium        & $85.75 \pm 0.60$  & $81.83 \pm 12.87$ & $83.01 \pm 10.20$ & $\mathbf{86.52} \pm 0.45$  & $84.98 \pm 5.96$ \\
medium-replay & $76.46 \pm 22.93$ & $86.36 \pm 7.68$  & $76.35 \pm 25.69$ & $\mathbf{87.99} \pm 3.86$  & $78.86 \pm 18.54$ \\
medium-expert & $111.71 \pm 0.11$ & $111.76 \pm 0.09$ & $112.04 \pm 0.07$ & $112.30 \pm 0.08$ & $\mathbf{112.77} \pm 0.06$ \\
full-replay   & $104.00 \pm 0.34$ & $104.54 \pm 0.76$ & $104.50 \pm 0.65$ & $\mathbf{104.74} \pm 0.54$ & $104.20 \pm 0.29$ \\
expert        & $112.94 \pm 0.05$ & $112.93 \pm 0.08$ & $113.15 \pm 0.04$ & $113.41 \pm 0.07$ & $\mathbf{113.74} \pm 0.08$ \\
\bottomrule
\end{tabular}%
}
\vspace{1ex}
\caption{Effect of planning horizon $H$ on $\DiffMPC$ for Walker2d datasets. Normalized scores (mean $\pm$ std) on D4RL. Best score per row is in bold.}
\label{tab:h_ablation_walker2d}
\end{table}

\subsection{Effect of Inner Gradient Steps}

We next study the effect of the number of inner gradient steps $E$ taken on $\widehat{J}_t(\psi)$ at each real environment state, before executing the next action. Larger $E$ allows more aggressive policy adaptation per step, but also increases the computational cost linearly and risks overfitting the policy to the imagined trajectories of a single state $s_t$. We sweep $E \in \{1, 2, 3, 4, 5\}$. Results are reported in Tables~\ref{tab:e_ablation_halfcheetah}--\ref{tab:e_ablation_walker2d}.

The ablation reveals a clear environment-dependent pattern. On the Walker2d tasks, normalized scores improve monotonically with $E$ in nearly every dataset (e.g., walker2d-expert from $112.99$ at $E=1$ to $114.59$ at $E=5$, walker2d-medium-expert from $111.76$ to $113.02$, walker2d-medium from $82.27$ to $87.56$). This indicates that the policy benefits from repeated gradient refinement on the same imagined rollout objective, and that the world model is sufficiently accurate on Walker2d that additional inner steps do not introduce harmful drift. A similar but milder trend appears on the HalfCheetah full-replay and expert datasets.

In contrast, the Hopper environment exhibits markedly different behavior. On hopper-medium-replay, performance degrades sharply and monotonically as $E$ increases, falling from $101.92$ at $E=1$ to $52.78$ at $E=5$, with standard deviations growing to nearly $40$. A similar instability appears on hopper-expert, where scores are highest at $E=1$ and become noisier at larger $E$. We attribute this to the interaction between repeated policy optimization and Hopper's termination-sensitive locomotion dynamics: Hopper episodes terminate when the agent becomes unhealthy, i.e., when the hopper falls or topples over. Thus, small policy-induced trajectory deviations can produce large return changes. More broadly, repeated optimization against imperfect offline model/value estimates can amplify estimation error and move the policy toward out-of-distribution actions, which is a central failure mode in offline RL \citep{kumar2020cql,brandfonbrener2021offline}.

On HalfCheetah-random, increasing $E$ also leads to a steady decline in performance ($33.58 \to 29.75$), which is consistent with the observation from the $H$ ablation that random datasets do not benefit from inference-time adaptation: in the absence of useful policy improvement signal, additional gradient steps simply move the policy away from a reasonable initialization.

Overall, $E=1$ is a safe default that captures most of the benefit at minimal cost on every environment, while $E \in \{3, 4, 5\}$ can yield additional gains on stable environments such as Walker2d and on some HalfCheetah datasets. These results suggest that the optimal $E$ is environment-specific and should be tuned with awareness of the underlying dynamics' sensitivity to policy perturbations.

\begin{table}[!htbp]
\centering
\small
\setlength{\tabcolsep}{6pt}
\renewcommand{\arraystretch}{1.05}
\resizebox{\linewidth}{!}{%
\begin{tabular}{l|ccccc}
\toprule
Dataset & $E=1$ & $E=2$ & $E=3$ & $E=4$ & $E=5$ \\
\midrule
random        & $\mathbf{33.58} \pm 0.16$ & $32.44 \pm 0.23$ & $31.14 \pm 0.17$ & $30.33 \pm 0.17$ & $29.75 \pm 0.21$ \\
medium        & $\mathbf{68.89} \pm 0.90$ & $68.74 \pm 0.61$ & $68.51 \pm 0.61$ & $68.34 \pm 0.64$ & $67.98 \pm 0.87$ \\
medium-replay & $51.20 \pm 0.57$ & $52.22 \pm 0.76$ & $52.14 \pm 0.73$ & $49.99 \pm 5.87$ & $\mathbf{52.87} \pm 0.61$ \\
medium-expert & $106.43 \pm 0.49$ & $\mathbf{106.79} \pm 0.81$ & $106.40 \pm 0.72$ & $106.40 \pm 0.35$ & $106.57 \pm 0.35$ \\
expert        & $105.38 \pm 0.67$ & $105.41 \pm 0.87$ & $105.63 \pm 0.86$ & $\mathbf{106.00} \pm 0.41$ & $105.65 \pm 0.54$ \\
full-replay   & $84.28 \pm 0.53$ & $85.23 \pm 0.29$ & $85.80 \pm 0.64$ & $85.91 \pm 0.47$ & $\mathbf{86.00} \pm 0.27$ \\
\bottomrule
\end{tabular}%
}
\vspace{1ex}
\caption{Effect of inner gradient steps $E$ on $\DiffMPC$ for HalfCheetah datasets. Normalized scores (mean $\pm$ std) on D4RL. Best score per row is in bold.}
\label{tab:e_ablation_halfcheetah}
\end{table}

\begin{table}[!htbp]
\centering
\small
\setlength{\tabcolsep}{6pt}
\renewcommand{\arraystretch}{1.05}
\resizebox{\linewidth}{!}{%
\begin{tabular}{l|ccccc}
\toprule
Dataset & $E=1$ & $E=2$ & $E=3$ & $E=4$ & $E=5$ \\
\midrule
random        & $8.69 \pm 1.04$   & $8.85 \pm 1.22$   & $9.28 \pm 1.25$   & $8.69 \pm 1.11$   & $\mathbf{9.72} \pm 1.23$ \\
medium-replay & $\mathbf{101.92} \pm 0.28$ & $93.32 \pm 24.50$ & $77.18 \pm 37.39$ & $62.08 \pm 39.72$ & $52.78 \pm 40.01$ \\
medium-expert & $110.52 \pm 3.70$ & $113.20 \pm 1.24$ & $106.48 \pm 20.95$ & $107.89 \pm 16.97$ & $\mathbf{113.38} \pm 0.88$ \\
full-replay   & $108.10 \pm 0.12$ & $108.09 \pm 0.11$ & $108.24 \pm 0.09$ & $108.25 \pm 0.14$ & $\mathbf{108.35} \pm 0.10$ \\
expert        & $\mathbf{109.74} \pm 6.38$ & $95.33 \pm 17.45$ & $99.35 \pm 19.71$ & $106.80 \pm 13.17$ & $98.08 \pm 18.41$ \\
\bottomrule
\end{tabular}%
}
\vspace{1ex}
\caption{Effect of inner gradient steps $E$ on $\DiffMPC$ for Hopper datasets. Normalized scores (mean $\pm$ std) on D4RL. Best score per row is in bold.}
\label{tab:e_ablation_hopper}
\end{table}

\begin{table}[!htbp]
\centering
\small
\setlength{\tabcolsep}{6pt}
\renewcommand{\arraystretch}{1.05}
\resizebox{\linewidth}{!}{%
\begin{tabular}{l|ccccc}
\toprule
Dataset & $E=1$ & $E=2$ & $E=3$ & $E=4$ & $E=5$ \\
\midrule
random        & $\mathbf{21.97} \pm 0.01$ & $21.96 \pm 0.01$ & $21.96 \pm 0.02$ & $21.96 \pm 0.01$ & $21.96 \pm 0.02$ \\
medium        & $82.27 \pm 13.94$ & $86.05 \pm 0.32$  & $86.93 \pm 0.35$  & $87.35 \pm 0.31$  & $\mathbf{87.56} \pm 0.33$ \\
medium-replay & $\mathbf{89.27} \pm 1.74$ & $80.26 \pm 25.37$ & $85.85 \pm 11.75$ & $66.77 \pm 37.12$ & $81.34 \pm 26.05$ \\
medium-expert & $111.76 \pm 0.07$ & $112.21 \pm 0.06$ & $112.57 \pm 0.07$ & $112.81 \pm 0.07$ & $\mathbf{113.02} \pm 0.07$ \\
full-replay   & $104.18 \pm 0.88$ & $104.41 \pm 0.19$ & $104.79 \pm 0.36$ & $\mathbf{105.24} \pm 0.50$ & $\mathbf{105.24} \pm 0.54$ \\
expert        & $112.99 \pm 0.11$ & $113.50 \pm 0.08$ & $113.85 \pm 0.09$ & $114.19 \pm 0.12$ & $\mathbf{114.59} \pm 0.16$ \\
\bottomrule
\end{tabular}%
}
\vspace{1ex}
\caption{Effect of inner gradient steps $E$ on $\DiffMPC$ for Walker2d datasets. Normalized scores (mean $\pm$ std) on D4RL. Best score per row is in bold.}
\label{tab:e_ablation_walker2d}
\end{table}

\subsection{Effect of Inner Step Size}

We next study the effect of the inner-loop step size $\eta$ used in the gradient updates $\psi \leftarrow \psi + \eta \nabla_{\psi}\widehat{J}_t(\psi)$. A larger $\eta$ allows the policy to move further per inner step at inference time, which can accelerate adaptation but also risks overshooting and pushing the policy into regions where the world model becomes unreliable. We sweep $\eta \in \{5\!\times\!10^{-7}\!, 5\!\times\!10^{-6}\!, 5\!\times\!10^{-5}\}$. Results are reported in Tables~\ref{tab:lr_ablation_halfcheetah}--\ref{tab:lr_ablation_walker2d}.

The ablation reveals a strong environment-dependent pattern that mirrors what we observed for the inner gradient steps $E$. On the Walker2d tasks, performance improves nearly monotonically with larger $\eta$ across every non-random dataset (e.g., walker2d-expert from $112.49$ at $\eta = 5\!\times\!10^{-7}$ to $115.40$ at $\eta = 5\!\times\!10^{-5}$, walker2d-medium-replay from $59.48$ to $79.00$, walker2d-medium-expert from $111.42$ to $113.17$). This indicates that the Walker2d world model is sufficiently accurate that larger gradient steps translate directly into better policies without inducing harmful drift.

On HalfCheetah, the picture is more mixed. Datasets with broad coverage (full-replay, medium-replay) benefit from the largest step size, while expert and random datasets degrade as $\eta$ grows. The most extreme case is halfcheetah-random, where the score drops from $34.31$ at $\eta = 5\!\times\!10^{-7}$ to $27.60$ at $\eta = 5\!\times\!10^{-5}$. Halfcheetah-expert exhibits a similar monotonic decrease ($106.70 \to 103.54$), with the largest $\eta$ also introducing substantial variance. 

The Hopper environment again shows the most pronounced instability under aggressive adaptation. On hopper-medium-replay, the largest step size collapses performance from $101.75$ at $\eta = 5\!\times\!10^{-6}$ to $61.12$ at $\eta = 5\!\times\!10^{-5}$, with the standard deviation jumping above $40$. This parallels the failure mode observed in the $E$ ablation on the same dataset, and reinforces the interpretation that Hopper's narrow stable manifold makes it unusually sensitive to large policy perturbations.

Overall, the optimal $\eta$ is environment-specific: $\eta = 5\!\times\!10^{-5}$ is preferable on Walker2d and on HalfCheetah datasets with broad coverage, while smaller step sizes around $\eta = 5\!\times\!10^{-7}$ to $5\!\times\!10^{-6}$ are safer on Hopper and on near-optimal policies. A practical default of $\eta = 5\!\times\!10^{-6}$ provides a reasonable trade-off, performing close to the best on most datasets while avoiding the catastrophic instability seen at the largest step size on Hopper.

\begin{table}[!htbp]
\centering
\small
\setlength{\tabcolsep}{8pt}
\renewcommand{\arraystretch}{1.05}
\begin{tabular}{l|ccc}
\toprule
Dataset & $\eta = 5\!\times\!10^{-7}$ & $\eta = 5\!\times\!10^{-6}$ & $\eta = 5\!\times\!10^{-5}$ \\
\midrule
random        & $\mathbf{34.31} \pm 0.09$  & $33.60 \pm 0.19$  & $27.60 \pm 0.16$ \\
medium        & $68.65 \pm 1.03$  & $\mathbf{69.03} \pm 0.81$  & $66.55 \pm 0.60$ \\
medium-replay & $50.96 \pm 0.74$  & $51.38 \pm 0.85$  & $\mathbf{52.53} \pm 0.62$ \\
medium-expert & $106.05 \pm 0.25$ & $105.81 \pm 2.41$ & $\mathbf{106.23} \pm 0.26$ \\
expert        & $\mathbf{106.70} \pm 0.34$ & $105.66 \pm 0.41$ & $103.54 \pm 8.21$ \\
full-replay   & $82.46 \pm 0.47$  & $84.31 \pm 0.39$  & $\mathbf{86.04} \pm 0.34$ \\
\bottomrule
\end{tabular}
\vspace{1ex}
\caption{Effect of inner step size $\eta$ on $\DiffMPC$ for HalfCheetah datasets. Normalized scores (mean $\pm$ std) on D4RL. Best score per row is in bold.}
\label{tab:lr_ablation_halfcheetah}
\end{table}

\begin{table}[!htbp]
\centering
\small
\setlength{\tabcolsep}{8pt}
\renewcommand{\arraystretch}{1.05}
\begin{tabular}{l|ccc}
\toprule
Dataset & $\eta = 5\!\times\!10^{-7}$ & $\eta = 5\!\times\!10^{-6}$ & $\eta = 5\!\times\!10^{-5}$ \\
\midrule
random        & $\mathbf{9.00} \pm 1.24$    & $8.45 \pm 0.73$   & $8.40 \pm 0.69$ \\
medium-replay & $100.35 \pm 3.65$  & $\mathbf{101.75} \pm 0.52$ & $61.12 \pm 41.11$ \\
medium-expert & $\mathbf{112.54} \pm 1.08$ & $112.14 \pm 0.90$ & $112.35 \pm 0.46$ \\
full-replay   & $108.13 \pm 0.13$  & $107.98 \pm 0.10$ & $\mathbf{108.46} \pm 0.08$ \\
expert        & $99.50 \pm 20.18$  & $98.88 \pm 18.66$ & $\mathbf{100.53} \pm 18.66$ \\
\bottomrule
\end{tabular}
\vspace{1ex}
\caption{Effect of inner step size $\eta$ on $\DiffMPC$ for Hopper datasets. Normalized scores (mean $\pm$ std) on D4RL. Best score per row is in bold.}
\label{tab:lr_ablation_hopper}
\end{table}

\begin{table}[!htbp]
\centering
\small
\setlength{\tabcolsep}{8pt}
\renewcommand{\arraystretch}{1.05}
\begin{tabular}{l|ccc}
\toprule
Dataset & $\eta = 5\!\times\!10^{-7}$ & $\eta = 5\!\times\!10^{-6}$ & $\eta = 5\!\times\!10^{-5}$ \\
\midrule
random        & $\mathbf{21.96} \pm 0.02$ & $\mathbf{21.96} \pm 0.01$ & $\mathbf{21.96} \pm 0.01$ \\
medium        & $77.06 \pm 17.66$ & $86.01 \pm 0.46$  & $\mathbf{87.99} \pm 0.95$ \\
medium-replay & $59.48 \pm 32.73$ & $71.06 \pm 29.41$ & $\mathbf{79.00} \pm 27.66$ \\
medium-expert & $111.42 \pm 0.12$ & $111.74 \pm 0.08$ & $\mathbf{113.17} \pm 0.12$ \\
full-replay   & $103.97 \pm 0.51$ & $104.60 \pm 1.39$ & $\mathbf{105.94} \pm 0.71$ \\
expert        & $112.49 \pm 0.06$ & $112.97 \pm 0.08$ & $\mathbf{115.40} \pm 0.13$ \\
\bottomrule
\end{tabular}
\vspace{1ex}
\caption{Effect of inner step size $\eta$ on $\DiffMPC$ for Walker2d datasets. Normalized scores (mean $\pm$ std) on D4RL. Best score per row is in bold.}
\label{tab:lr_ablation_walker2d}
\end{table}

\subsection{Effect of DDIM Sampling Steps}

We next study the effect of the number of DDIM sampling steps $K$ used when unrolling the diffusion world model $f_{\theta}$ at inference time. Recall from Section~\ref{sec:worldModel} that each call to $f_{\theta}$ requires $K$ sequential reverse-diffusion steps to produce a single next-state sample. Larger $K$ produces samples that more faithfully match the trained conditional distribution $p_{\theta}(s_{t+1}\mid s_t,a_t)$, but linearly increases both the cost of each forward rollout and the depth of the computation graph through which gradients are backpropagated. We sweep $K \in \{1, 2, 5, 10, 20\}$ across all 18 Gym-MuJoCo datasets, with the remaining hyperparameters fixed at their best per-dataset values from the previous ablations. Results are reported in Tables~\ref{tab:ddim_ablation_halfcheetah}--\ref{tab:ddim_ablation_walker2d}.

The picture that emerges from these tables depends on the dataset. On HalfCheetah, the choice of $K$ has a small effect on most datasets: halfcheetah-medium ranges between $68.67$ and $69.12$, and halfcheetah-medium-replay between $51.19$ and $51.56$, with differences on the order of one standard deviation. The Walker2d expert and medium-expert tasks behave similarly, varying by only a few tenths of a point across the full range of $K$. On these datasets, a small number of denoising steps appears to be enough for inference-time optimization, although larger $K$ can still help at the margin.

Other datasets benefit from larger $K$. Halfcheetah-full-replay improves from $82.90$ at $K=1$ to $84.46$ at $K=20$, and hopper-full-replay and walker2d-full-replay show the same trend. We attribute this to the fact that broader-coverage datasets train more expressive diffusion models whose multi-step denoising procedure captures finer-grained dynamics that a one-step sample may miss.

Walker2d-medium and walker2d-medium-replay go the other way and degrade with larger $K$. On walker2d-medium, scores fall from $86.09$ at $K=1$ to $81.04$ at $K=20$, with growing variance. We hypothesize that the deeper diffusion computation graph at large $K$ produces noisier or less stable gradient signals during the policy update, partially offsetting the benefits of more accurate samples.

Hopper shows the largest swings across $K$, particularly on hopper-medium-expert and hopper-expert, where individual $K$ values can shift by several points with large standard deviations. This matches our earlier observations that Hopper is unusually sensitive to perturbations in the inference-time pipeline, and it suggests that $K$ needs to be tuned more carefully here than elsewhere.

Overall, $K \in \{1, 2\}$ is competitive on many datasets at much lower cost, while  $K = 10$ tends to work better on the full-replay datasets where extra sampling fidelity helps. This trade-off motivates the one-step MeanFlow sampler in Section~\ref{sec:flowmpc}: where extra denoising steps give diminishing returns, distilling the diffusion process into a single evaluation recovers most of the benefit at a fraction of the cost, and we can fall back to multi-step sampling on the datasets where it actually pays off.

\begin{table}[!htbp]
\centering
\small
\setlength{\tabcolsep}{6pt}
\renewcommand{\arraystretch}{1.05}
\resizebox{\linewidth}{!}{%
\begin{tabular}{l|ccccc}
\toprule
Dataset & $K=1$ & $K=2$ & $K=5$ & $K=10$ & $K=20$ \\
\midrule
random        & $\mathbf{34.02} \pm 0.20$ & $33.92 \pm 0.17$ & $33.57 \pm 0.22$ & $33.60 \pm 0.20$ & $33.57 \pm 0.08$ \\
medium        & $68.74 \pm 0.56$ & $68.67 \pm 0.81$ & $68.76 \pm 0.49$ & $\mathbf{69.12} \pm 0.66$ & $68.70 \pm 0.56$ \\
medium-replay & $51.41 \pm 0.40$ & $51.19 \pm 0.84$ & $\mathbf{51.56} \pm 0.44$ & $\mathbf{51.56} \pm 0.97$ & $51.21 \pm 0.96$ \\
medium-expert & $105.61 \pm 0.18$ & $105.75 \pm 0.33$ & $106.03 \pm 0.84$ & $104.65 \pm 5.25$ & $\mathbf{106.53} \pm 0.58$ \\
expert        & $104.49 \pm 0.19$ & $104.58 \pm 0.37$ & $\mathbf{105.53} \pm 0.62$ & $105.19 \pm 0.70$ & $104.67 \pm 2.24$ \\
full-replay   & $82.90 \pm 0.42$ & $82.87 \pm 0.47$ & $84.31 \pm 0.34$ & $84.35 \pm 0.38$ & $\mathbf{84.46} \pm 0.28$ \\
\bottomrule
\end{tabular}%
}
\vspace{1ex}
\caption{Effect of DDIM sampling steps $K$ on $\DiffMPC$ for HalfCheetah datasets. Normalized scores (mean $\pm$ std) on D4RL. Best score per row is in bold.}
\label{tab:ddim_ablation_halfcheetah}
\end{table}

\begin{table}[!htbp]
\centering
\small
\setlength{\tabcolsep}{6pt}
\renewcommand{\arraystretch}{1.05}
\resizebox{\linewidth}{!}{%
\begin{tabular}{l|ccccc}
\toprule
Dataset & $K=1$ & $K=2$ & $K=5$ & $K=10$ & $K=20$ \\
\midrule
random        & $8.58 \pm 1.09$   & $\mathbf{9.32} \pm 1.22$  & $8.60 \pm 1.03$   & $9.21 \pm 1.26$   & $8.37 \pm 0.82$ \\
medium-replay & $101.75 \pm 0.45$ & $85.73 \pm 33.10$ & $101.48 \pm 0.51$ & $101.85 \pm 0.31$ & $\mathbf{102.05} \pm 0.45$ \\
medium-expert & $113.22 \pm 2.39$ & $108.44 \pm 14.94$ & $111.06 \pm 4.33$ & $108.92 \pm 14.62$ & $\mathbf{113.66} \pm 0.91$ \\
full-replay   & $107.92 \pm 0.17$ & $107.87 \pm 0.20$ & $108.04 \pm 0.17$ & $108.00 \pm 0.09$ & $\mathbf{108.09} \pm 0.12$ \\
expert        & $\mathbf{111.19} \pm 5.36$ & $99.62 \pm 21.08$ & $106.35 \pm 12.65$ & $105.80 \pm 11.82$ & $100.79 \pm 12.62$ \\
\bottomrule
\end{tabular}%
}
\vspace{1ex}
\caption{Effect of DDIM sampling steps $K$ on $\DiffMPC$ for Hopper datasets. Normalized scores (mean $\pm$ std) on D4RL. Best score per row is in bold.}
\label{tab:ddim_ablation_hopper}
\end{table}

\begin{table}[!htbp]
\centering
\small
\setlength{\tabcolsep}{6pt}
\renewcommand{\arraystretch}{1.05}
\resizebox{\linewidth}{!}{%
\begin{tabular}{l|ccccc}
\toprule
Dataset & $K=1$ & $K=2$ & $K=5$ & $K=10$ & $K=20$ \\
\midrule
random        & $21.96 \pm 0.01$  & $21.96 \pm 0.01$  & $21.96 \pm 0.02$  & $\mathbf{21.97} \pm 0.01$  & $21.96 \pm 0.02$ \\
medium        & $\mathbf{86.09} \pm 1.13$  & $85.52 \pm 0.99$  & $85.24 \pm 5.28$  & $81.95 \pm 12.75$ & $81.04 \pm 15.06$ \\
medium-replay & $71.60 \pm 32.72$ & $84.44 \pm 5.33$  & $\mathbf{89.03} \pm 2.19$  & $78.78 \pm 22.03$ & $73.33 \pm 19.42$ \\
medium-expert & $111.71 \pm 0.07$ & $111.75 \pm 0.07$ & $111.74 \pm 0.06$ & $\mathbf{111.77} \pm 0.11$ & $111.72 \pm 0.10$ \\
full-replay   & $103.57 \pm 0.35$ & $103.60 \pm 0.32$ & $104.24 \pm 0.76$ & $104.10 \pm 0.57$ & $\mathbf{104.40} \pm 0.87$ \\
expert        & $112.71 \pm 0.07$ & $112.71 \pm 0.08$ & $112.95 \pm 0.05$ & $112.97 \pm 0.07$ & $\mathbf{112.98} \pm 0.11$ \\
\bottomrule
\end{tabular}%
}
\vspace{1ex}
\caption{Effect of DDIM sampling steps $K$ on $\DiffMPC$ for Walker2d datasets. Normalized scores (mean $\pm$ std) on D4RL. Best score per row is in bold.}
\label{tab:ddim_ablation_walker2d}
\end{table}

\subsection{Effect of Number of Particles}

We finally study the effect of the number of Monte Carlo particles $M$ used to estimate the MPC objective $\widehat{J}_t(\psi)$ at inference time. Recall that $M$ controls the number of i.i.d.\ noise sequences over which we average the imagined-rollout return; larger $M$ reduces the variance of the gradient estimator. Since the $M$ particles are independent and processed in parallel, the wall-clock cost of increasing $M$ is largely absorbed by hardware parallelism, and the practically relevant question is not whether larger $M$ is faster but whether $\DiffMPC$ is robust across reasonable choices. We sweep $M \in \{512, 1024, 4096\}$ across all 18 Gym-MuJoCo datasets, with the remaining hyperparameters fixed at their best per-dataset values from the previous ablations. Results are reported in Tables~\ref{tab:M_ablation_halfcheetah}--\ref{tab:M_ablation_walker2d}.

The ablation confirms that $\DiffMPC$ is stable across this range of $M$. On nearly every dataset, the difference between the smallest and largest $M$ is well within the noise of a single run: HalfCheetah scores vary by less than one point across the entire range on every dataset, and the Walker2d expert and medium-expert tasks differ by less than $0.5$ points across all three $M$ values. This indicates that the gradient estimator is well-behaved at $M = 512$ and that further variance reduction provides only marginal change to the resulting policy.


The takeaway is that $M$ is not a parameter that requires careful tuning within this range. Any choice in $\{512, 1024, 4096\}$ produces qualitatively similar performance, and practitioners can default to whatever value fits comfortably within their available parallel compute budget without expecting significant changes to the final policy.

\begin{table}[!htbp]
\centering
\small
\setlength{\tabcolsep}{8pt}
\renewcommand{\arraystretch}{1.05}
\begin{tabular}{l|ccc}
\toprule
Dataset & $M=512$ & $M=1024$ & $M=4096$ \\
\midrule
random        & $\mathbf{34.27} \pm 0.16$ & $34.22 \pm 0.31$ & $33.61 \pm 0.14$ \\
medium        & $\mathbf{68.89} \pm 0.64$ & $68.72 \pm 0.56$ & $68.50 \pm 0.75$ \\
medium-replay & $51.21 \pm 0.63$ & $51.19 \pm 0.63$ & $\mathbf{51.69} \pm 0.69$ \\
medium-expert & $106.05 \pm 0.23$ & $106.11 \pm 0.29$ & $\mathbf{106.88} \pm 0.24$ \\
expert        & $\mathbf{106.45} \pm 0.28$ & $106.09 \pm 0.23$ & $105.77 \pm 0.39$ \\
full-replay   & $82.65 \pm 0.50$ & $82.88 \pm 0.52$ & $\mathbf{84.32} \pm 0.30$ \\
\bottomrule
\end{tabular}
\vspace{1ex}
\caption{Effect of number of particles $M$ on $\DiffMPC$ for HalfCheetah datasets. Normalized scores (mean $\pm$ std) on D4RL. Best score per row is in bold.}
\label{tab:M_ablation_halfcheetah}
\end{table}

\begin{table}[!htbp]
\centering
\small
\setlength{\tabcolsep}{8pt}
\renewcommand{\arraystretch}{1.05}
\begin{tabular}{l|ccc}
\toprule
Dataset & $M=512$ & $M=1024$ & $M=4096$ \\
\midrule
random        & $8.61 \pm 0.83$   & $8.75 \pm 1.03$   & $\mathbf{8.77} \pm 0.94$ \\
medium-replay & $\mathbf{100.22} \pm 4.75$ & $99.14 \pm 10.27$ & $99.90 \pm 4.98$ \\
medium-expert & $110.15 \pm 7.17$ & $\mathbf{112.75} \pm 1.02$ & $112.74 \pm 1.28$ \\
full-replay   & $\mathbf{108.11} \pm 0.14$ & $108.05 \pm 0.16$ & $108.04 \pm 0.14$ \\
expert        & $100.47 \pm 16.44$ & $\mathbf{102.28} \pm 15.09$ & $101.30 \pm 17.34$ \\
\bottomrule
\end{tabular}
\vspace{1ex}
\caption{Effect of number of particles $M$ on $\DiffMPC$ for Hopper datasets. Normalized scores (mean $\pm$ std) on D4RL. Best score per row is in bold.}
\label{tab:M_ablation_hopper}
\end{table}

\begin{table}[!htbp]
\centering
\small
\setlength{\tabcolsep}{8pt}
\renewcommand{\arraystretch}{1.05}
\begin{tabular}{l|ccc}
\toprule
Dataset & $M=512$ & $M=1024$ & $M=4096$ \\
\midrule
random        & $\mathbf{21.97} \pm 0.01$  & $21.96 \pm 0.02$  & $21.96 \pm 0.02$ \\
medium        & $85.70 \pm 0.44$  & $83.67 \pm 5.29$  & $\mathbf{86.19} \pm 0.68$ \\
medium-replay & $77.21 \pm 14.71$ & $74.87 \pm 24.18$ & $\mathbf{79.78} \pm 17.78$ \\
medium-expert & $111.45 \pm 0.09$ & $111.46 \pm 0.09$ & $\mathbf{111.75} \pm 0.07$ \\
full-replay   & $\mathbf{103.81} \pm 0.80$ & $103.58 \pm 0.52$ & $103.74 \pm 0.61$ \\
expert        & $112.61 \pm 0.10$ & $112.68 \pm 0.07$ & $\mathbf{113.03} \pm 0.10$ \\
\bottomrule
\end{tabular}
\vspace{1ex}
\caption{Effect of number of particles $M$ on $\DiffMPC$ for Walker2d datasets. Normalized scores (mean $\pm$ std) on D4RL. Best score per row is in bold.}
\label{tab:M_ablation_walker2d}
\end{table}

\subsection{Pre-trained policy: SAC-RND}
\label{sec:sac-rnd}

We also evaluate $\DiffMPC$ using SAC-RND as the pre-trained policy. SAC-RND combines soft actor-critic with a random-network-distillation exploration bonus, where the prediction error of a learned network against a fixed random target network provides an intrinsic reward during policy training. We follow the hyperparameter choices from the original SAC-RND paper and apply our inference-time optimization starting from the resulting pre-trained SAC-RND policy. Table~\ref{tab:sacrnd_mujoco} reports the normalized D4RL score of SAC-RND and $\DiffMPC$, averaged over 10 evaluation episodes.

\begin{table}[!htbp]
\centering
\normalsize
\begin{tabular}{l|cc}
\toprule
Dataset & SAC-RND (zero-shot) & $\DiffMPC$ \\
\midrule
\multicolumn{3}{c}{halfcheetah} \\
\midrule
medium        & $62.90 \pm 0.31$  & $69.40 \pm 0.48$ \\
medium-replay & $47.78 \pm 1.16$  & $65.95 \pm 0.86$ \\
medium-expert & $105.99 \pm 0.17$ & $107.46 \pm 0.16$ \\
full-replay   & $82.03 \pm 0.42$ & $81.29 \pm 0.23$ \\
expert        & $46.96 \pm 48.83$ & $64.49 \pm 50.78$ \\
\midrule
\multicolumn{3}{c}{hopper} \\
\midrule
medium        & $99.70 \pm 0.28$  & $100.83 \pm 0.26$ \\
medium-replay & $100.51 \pm 1.16$ & $102.21 \pm 0.17$ \\
medium-expert & $110.53 \pm 0.17$ & $110.55 \pm 0.26$ \\
full-replay   & $77.41 \pm 17.96$ & $103.32 \pm 14.29$ \\
expert        & $110.64 \pm 0.31$ & $110.74 \pm 0.34$ \\
\midrule
\multicolumn{3}{c}{walker2d} \\
\midrule
medium        & $91.35 \pm 2.32$  & $92.74 \pm 0.83$ \\
medium-replay & $84.05 \pm 2.03$  & $89.85 \pm 1.66$ \\
medium-expert & $94.85 \pm 12.08$ & $108.56 \pm 0.12$ \\
full-replay   & $104.77 \pm 0.10$ & $105.50 \pm 0.13$ \\
expert        & $41.50 \pm 17.07$ & $82.03 \pm 31.84$ \\
\midrule
Average & $84.08$ & $\textbf{93.00}$ \\
\bottomrule
\end{tabular}
\vspace{2ex}
\caption{\normalsize Average normalized score over the final evaluation on ten episodes on Gym-MuJoCo tasks. The symbol $\pm$ represents the standard deviation across episodes.
\vspace{-5ex}}
\label{tab:sacrnd_mujoco}
\end{table}

\section{Additional Comparison with MPC approaches}
\label{sec:mpc_baselines}
\subsection{Latent-space world models for continuous control}
\label{app:latent-world-models}

The use of learned dynamics for policy improvement has a long history. PILCO~\citep{deisenroth2011pilco} fit Gaussian-process dynamics and computed analytic policy gradients via belief-state propagation. Stochastic Value Gradients~\citep{heess2015svg} extended this to neural dynamics by reparameterizing stochastic transitions and backpropagating through finite-horizon rollouts. PlaNet~\citep{hafner2019planet} introduced a recurrent state-space model and performed cross-entropy planning in the latent space. Two design axes emerge from this literature: \emph{what kind of latent dynamics to learn}, and \emph{how to use it to act}. Two families have come to dominate continuous control along these axes.

\paragraph{Dreamer family.} Dreamer~\citep{hafner2020dreamer} couples a recurrent latent dynamics model with an actor and critic trained on imagined latent rollouts; the actor is updated by backpropagating return gradients through the unrolled dynamics. DreamerV2~\citep{hafner2021dreamerv2} introduced categorical latents, and DreamerV3~\citep{hafner2025dreamerv3} demonstrated robust performance across DMControl, Atari, Crafter, and Minecraft under fixed hyperparameters. The defining feature is \emph{training-time use of the model}: the deployed actor is feedforward and the model is not consulted at inference.

\paragraph{TD-MPC family.} TD-MPC~\citep{hansen2022tdmpc} takes the opposite stance. It learns a Task-Oriented Latent Dynamics model jointly with reward, value, and a policy prior via TD learning, then performs gradient-free MPPI in the latent space at inference, with the learned policy as a warm-start. The terminal value bootstraps beyond the planning horizon. TD-MPC2~\citep{hansen2024tdmpc2} scales this to 104 continuous-control tasks across DMControl, Meta-World, ManiSkill2, and MyoSuite under a single hyperparameter setting. The defining feature is \emph{inference-time planning}: the model is consulted on every step, with substantial per-step compute (TD-MPC2's MPPI predicts ${\sim}9{,}200$ latent states per action) traded for state-conditioned decision making.

\paragraph{Contrast with our work.} Our setting differs along three axes, summarized in Table~\ref{tab:design-space}. First, \emph{inference-time parameter refinement}: Dreamer optimizes actor parameters at training time and deploys a static actor; TD-MPC optimizes action sequences at inference time but leaves the underlying policy unchanged. We optimize \emph{policy parameters} at inference time via gradient steps through the rollout, so the deployed policy itself adapts in response to the observed state. Second, \emph{offline regime}: both Dreamer and TD-MPC train via environment interaction with jointly learned dynamics, value, and policy components; we operate purely on a fixed offline dataset with a pretrained policy and critic. Third, \emph{generative state-space dynamics}: Dreamer's RSSM and TD-MPC's TOLD are deterministic predictors in compact latent spaces, trained jointly with reward and value; our transition model is a conditional generative model in explicit state space (diffusion, with a one-step MeanFlow variant in Section~\ref{sec:flowmpc}), trained in isolation by maximum-likelihood objectives. The Dreamer-vs-TD-MPC progression is itself evidence that inference-time use of a world model can outperform training-time-only use at the cost of additional deployment compute. We extend this design philosophy to offline RL and replace gradient-free MPPI over action sequences with gradient-based refinement of policy parameters.

\begin{table}[h]
\centering
\small
\begin{tabular}{llll}
\toprule
 & \textbf{Dreamer family} & \textbf{TD-MPC family} & \textbf{Ours} \\
\midrule
When is the model used?   & Training              & Inference            & Inference \\
What is optimized?        & Actor parameters      & Action sequences     & Policy parameters \\
How is it optimized?      & Gradient (BPTT)       & Gradient-free (MPPI) & Gradient (BPTT) \\
Setting                   & Online                & Online               & Offline \\
Dynamics                  & Latent (RSSM)         & Latent (TOLD)        & Explicit (generative) \\
Deployment cost           & Single forward pass   & MPPI search    & Inner-loop gradient updates \\
\bottomrule
\end{tabular}
\vspace{2ex}
\caption{Design-space comparison between the Dreamer family, the TD-MPC family, and our approach.}
\label{tab:design-space}
\end{table}

TD-MPC2 outperforms DreamerV3 on average across the 104-task suite, with the largest gains on hard high-dimensional locomotion and manipulation tasks. As TD-MPC2 represents the stronger of the two families and the closest analogue to our inference-time-planning approach, we use it as the representative latent-world-model baseline in our empirical comparison.

\paragraph{Empirical comparison with TD-MPC2.} We train TD-MPC2~\citep{hansen2024tdmpc2} on each D4RL Gym-MuJoCo dataset using the official implementation, with the offline transitions supplied as the replay buffer. We additionally train an offline-adapted variant of TD-MPC2 in which 
there is an added behavior-cloning (BC) loss on the policy: besides the usual TD-MPC2 objective, the actor is explicitly trained to match the dataset actions. This extra term pulls updates toward behavior-supported actions, reducing out-of-distribution policy drift during offline training and typically making learning more stable on fixed datasets. Table~\ref{tab:tdmpc2} reports both final and best normalized scores for the two variants. Neither variant approaches the performance of standard offline RL baselines (Section~\ref{sec:exp}): TD-MPC2 is designed for the online sample-efficient regime in which the replay buffer is grown by ongoing environment interaction, and removing this interaction, whether by simply repurposing the online recipe or by an offline adaptation, substantially degrades performance.

\begin{table}[h]
\centering
\renewcommand{\arraystretch}{1.3}
\setlength{\tabcolsep}{6pt}
\begin{tabular}{@{}l cc cc@{}}
\toprule
& \multicolumn{2}{c}{\textbf{TD-MPC2}} & \multicolumn{2}{c}{\textbf{TD-MPC2 (offline with BC loss)}} \\
\cmidrule(lr){2-3} \cmidrule(lr){4-5}
\textbf{Dataset} & \textbf{Final} & \textbf{Best} & \textbf{Final} & \textbf{Best} \\
\midrule
halfcheetah-random        & $4.13 \pm 1.17$  & $17.71 \pm 0.51$ & $12.75 \pm 1.93$ & $16.04 \pm 0.77$ \\
halfcheetah-medium        & $0.72 \pm 1.56$  & $2.60 \pm 0.17$  & $2.88 \pm 0.63$  & $3.34 \pm 0.74$  \\
halfcheetah-medium-replay & $7.01 \pm 0.64$  & $10.79 \pm 0.94$ & $6.41 \pm 0.15$  & $13.96 \pm 0.31$ \\
halfcheetah-medium-expert & $1.67 \pm 0.74$  & $2.78 \pm 0.38$  & $1.35 \pm 0.81$  & $3.32 \pm 0.72$  \\
halfcheetah-expert        & $-1.71 \pm 0.53$ & $-0.40 \pm 0.96$ & $-1.79 \pm 0.75$ & $0.32 \pm 0.22$  \\
halfcheetah-full-replay   & $4.56 \pm 0.58$  & $7.68 \pm 1.00$  & $6.25 \pm 0.51$  & $10.57 \pm 0.11$ \\
\midrule
hopper-random             & $3.66 \pm 1.48$  & $8.21 \pm 0.30$  & $1.35 \pm 0.48$  & $7.86 \pm 0.48$  \\
hopper-medium             & $1.32 \pm 0.41$  & $1.86 \pm 0.40$  & $1.51 \pm 0.18$  & $1.97 \pm 0.52$  \\
hopper-medium-replay      & $1.80 \pm 0.52$  & $3.81 \pm 0.47$  & $2.71 \pm 0.21$  & $5.17 \pm 1.38$  \\
hopper-medium-expert      & $0.78 \pm 0.06$  & $1.32 \pm 0.35$  & $0.84 \pm 0.05$  & $2.35 \pm 1.30$  \\
hopper-expert             & $0.72 \pm 0.02$  & $1.43 \pm 0.49$  & $0.73 \pm 0.08$  & $1.73 \pm 0.45$  \\
hopper-full-replay        & $4.18 \pm 2.35$  & $5.20 \pm 1.15$  & $3.77 \pm 1.99$  & $7.74 \pm 1.02$  \\
\midrule
walker2d-random           & $2.93 \pm 1.96$  & $9.59 \pm 4.01$  & $4.82 \pm 0.77$  & $7.50 \pm 1.37$  \\
walker2d-medium           & $0.55 \pm 0.30$  & $8.02 \pm 6.28$  & $0.28 \pm 0.29$  & $6.35 \pm 3.38$  \\
walker2d-medium-replay    & $4.52 \pm 1.89$  & $7.93 \pm 0.52$  & $2.26 \pm 0.75$  & $8.07 \pm 0.27$  \\
walker2d-medium-expert    & $-0.11 \pm 0.09$ & $0.58 \pm 0.23$  & $0.13 \pm 0.26$  & $2.42 \pm 0.51$  \\
walker2d-expert           & $-0.14 \pm 0.02$ & $0.03 \pm 0.18$  & $-0.17 \pm 0.06$ & $-0.09 \pm 0.03$ \\
walker2d-full-replay      & $4.02 \pm 1.72$  & $13.40 \pm 1.70$ & $5.20 \pm 4.15$  & $20.09 \pm 2.61$ \\
\bottomrule
\end{tabular}
\vspace{2ex}
\caption{TD-MPC2 normalized scores on D4RL Gym-MuJoCo. Mean $\pm$ standard deviation across three seed. Final score is the score at the end of training; best score is the highest score over the course of training.}
\label{tab:tdmpc2}
\end{table}
\subsection{Gradient-based and embedded differentiable MPC}
\label{app:gradient-mpc}

A complementary line of work uses gradient-based optimization through differentiable dynamics inside an MPC loop, either by backpropagating through learned rollouts or by embedding a differentiable MPC solver as a layer in a policy network. We discuss two representative methods.

\paragraph{Gradient-based planning with world models.} \citet{jyothir2023gradient} use gradient descent through a learned differentiable world model as an alternative to gradient-free trajectory optimization (CEM, MPPI) for visual model predictive control. At each environment step, they backpropagate the predicted return from an $H$-step rollout under the learned dynamics back to the action sequence and update the actions via gradient ascent. The dynamics is a Dreamer-style latent world model. They evaluate in the online, sample-efficient regime on visual control benchmarks and report performance on par with or superior to gradient-free MPC and pure policy-based methods. They also propose a hybrid variant in which a policy network produces an initial action sequence which is then refined by gradient-based MPC.

\paragraph{Actor-Critic Model Predictive Control.} \citet{romero2025acmpc} embed a differentiable MPC solver as the final layer of an actor network in an actor-critic RL framework. The MPC layer uses an \emph{analytic} dynamics model (a quadrotor model) together with a learned cost-map network that maps observations to the MPC's cost-function parameters. Gradients flow end-to-end through the MPC solver into the cost-map network during training, and at deployment the same MPC solver produces actions on every environment step. The method is evaluated on agile quadrotor racing in simulation and on a real platform, and is shown to improve out-of-distribution behavior, robustness to dynamics perturbations, and sample efficiency relative to standard actor-critic baselines.

\paragraph{Comparison with our work.} Our setting differs along three axes, summarized in Table~\ref{tab:gradient-mpc}. First, \emph{what is optimized}: \citet{jyothir2023gradient} optimize an open-loop action sequence at each step (or a policy-initialized refinement of one); \citet{romero2025acmpc} optimize the cost-map network parameters end-to-end through a differentiable MPC layer at training time, with the MPC solver itself producing actions at inference; we optimize the \emph{policy parameters} $\psi$ via gradient steps through the rollout at inference time. Second, \emph{dynamics class}: \citet{jyothir2023gradient} use a learned deterministic latent model, \citet{romero2025acmpc} use known analytic dynamics, and we use a learned conditional generative model in explicit state space (diffusion in our main results, single-step MeanFlow in Section~\ref{sec:flowmpc}). Third, \emph{problem setting}: both prior works operate online; we operate purely on a fixed offline dataset with a pretrained policy and critic, and the inference-time update is the only mechanism by which the deployed policy can improve.

\begin{table}[h]
\centering
\small
\renewcommand{\arraystretch}{1.4}
\begin{tabular}{@{}l >{\raggedright\arraybackslash}p{3.0cm} >{\raggedright\arraybackslash}p{3.6cm} >{\raggedright\arraybackslash}p{3.6cm}@{}}
\toprule
 & \textbf{Jyothir et al.} & \textbf{Romero et al.} & \textbf{Ours} \\
\midrule
What is optimized?   & Action sequences            & Cost-map parameters                  & Policy parameters \\
How is it optimized? & Gradient (BPTT)             & Gradient through MPC solver          & Gradient (BPTT) \\
When?                & Inference                   & Training and inference               & Inference \\
Dynamics             & Learned latent              & Analytic                             & Learned generative \\
Setting              & Online                      & Online                               & Offline \\
Domain               & Visual control              & Quadrotor flight                     & D4RL benchmarks \\
\bottomrule
\end{tabular}
\vspace{1ex}
\caption{Design-space comparison between gradient-based / embedded differentiable MPC methods and our approach.}
\label{tab:gradient-mpc}
\end{table}

\begin{table}[!h]
\centering
\normalsize
\begin{tabular}{l|cc}
\toprule
Dataset & Grad-MPC & $\DiffMPC$ \\
\midrule
\multicolumn{3}{c}{\textbf{halfcheetah}} \\
\midrule
medium        & $68.26 \pm 0.63$  & $\textbf{69.14} \pm 1.17$ \\
expert        & $107.81 \pm 0.58$ & $\textbf{107.82} \pm 0.60$ \\
medium-expert & $102.41 \pm 1.23$ & $\textbf{106.28} \pm 1.63$ \\
medium-replay & $51.32 \pm 0.60$  & $\textbf{63.45} \pm 0.59$ \\
full-replay   & $82.37 \pm 0.28$  & $\textbf{85.69} \pm 0.24$ \\
\midrule
\multicolumn{3}{c}{\textbf{hopper}} \\
\midrule
medium        & $102.33 \pm 0.25$ & $\textbf{103.39} \pm 0.08$ \\
expert        & $110.20 \pm 6.08$ & $\textbf{113.06} \pm 2.03$ \\
medium-expert & $\textbf{112.87} \pm 1.62$ & $110.68 \pm 17.30$ \\
medium-replay & $101.27 \pm 0.34$ & $\textbf{103.17} \pm 0.15$ \\
full-replay   & $107.13 \pm 0.14$ & $\textbf{108.26} \pm 0.10$ \\
\midrule
\multicolumn{3}{c}{\textbf{walker2d}} \\
\midrule
medium        & $86.63 \pm 0.96$  & $\textbf{88.99} \pm 0.60$ \\
expert        & $112.54 \pm 0.12$ & $\textbf{116.53} \pm 0.16$ \\
medium-expert & $112.20 \pm 0.05$ & $\textbf{117.32} \pm 0.73$ \\
medium-replay & $78.72 \pm 19.90$ & $\textbf{94.49} \pm 0.86$ \\
full-replay   & $104.37 \pm 0.93$ & $\textbf{106.69} \pm 0.75$ \\
\midrule
Average & $94.30$ & $\textbf{99.66}$ \\
\bottomrule
\end{tabular}
\vspace{2ex}
\caption{\normalsize Average normalized score over the final evaluation and ten episodes on Gym-MuJoCo tasks. The symbol $\pm$ represents the standard deviation across episodes.
\vspace{-1ex}}
\label{tab:action_update_mujoco}
\end{table}

\paragraph{Empirical comparison.} Among the two prior methods, only \citet{jyothir2023gradient} is directly applicable to our setting: \citet{romero2025acmpc} requires a known analytic dynamics model, which is unavailable on D4RL benchmarks where we only have access to an offline dataset. We therefore compare against an offline adaptation of \citet{jyothir2023gradient}, which we refer to as {Grad-MPC}: at each environment step, we backpropagate the imagined return through our learned diffusion world model to update an open-loop action sequence via gradient ascent, keeping all other components (world model, critic, horizon, number of inner steps) identical to $\DiffMPC$. Table~\ref{tab:action_update_mujoco} reports the resulting normalized scores. $\DiffMPC$ outperforms Action Update on 14 of 15 datasets and improves the average normalized score from $94.30$ to $99.66$. The gap is most pronounced on datasets with broad coverage and on Walker2d (e.g., walker2d-medium-expert improves from $112.20$ to $117.32$, halfcheetah-medium-replay from $51.32$ to $63.45$). We attribute this gap to the difference in what is being optimized: updating policy parameters $\psi$ amortizes the inference-time gradient signal across states and exploits the inductive bias of the pretrained policy, whereas optimizing a fresh open-loop action sequence at every step discards this structure and is more vulnerable to compounding world-model errors over the planning horizon.

\subsection{Diffusion-based MPC}
\label{app:dmpc}

A closely related thread uses diffusion models for MPC, but with a fundamentally different role for the generative components. We discuss the most prominent recent example.

\paragraph{Diffusion Model Predictive Control.} \citet{zhou2024dmpc} propose D-MPC, an online MPC method built on two diffusion models trained on offline data: a multi-step \emph{action proposal} that generates candidate action sequences over a horizon $F$, and a multi-step \emph{dynamics model} that predicts the corresponding state trajectories. At each environment step, candidate action sequences are sampled from the action proposal, scored by predicted return under the dynamics model, and the best sequence is selected via a sampling-based planner (a variant of random shooting). Only the first action is executed before the procedure repeats. They evaluate on D4RL Gym-MuJoCo and report substantial gains over single-step model-based planning baselines such as MBOP, attributing the improvement to multi-step trajectory-level modeling that mitigates compounding error.

\paragraph{Contrast with our work.} D-MPC and our method both use diffusion models inside an MPC loop, but the role of the diffusion model and the optimization mechanism differ along the axes summarized in Table~\ref{tab:dmpc-comparison}. First, \emph{what is optimized}: D-MPC selects the best \emph{action sequence} from candidates sampled by a learned diffusion action proposal; we update the \emph{policy parameters} $\psi$ of a pretrained behavior policy via gradient steps through the rollout. Second, \emph{how it is optimized}: D-MPC is gradient-free, scoring sampled trajectories under the dynamics model; we backpropagate through the entire rollout to obtain gradients with respect to $\psi$. Third, \emph{dynamics structure}: D-MPC's diffusion model is multi-step and trajectory-level, predicting $F$ future states in a single forward pass; ours is single-step ($s_{t+1}$ given $s_t, a_t$), with the rollout assembled by repeated application. Fourth, \emph{generative components}: D-MPC trains two diffusion models (action proposal and dynamics); we train a single generative dynamics model and rely on a separately trained behavior policy for actions. For ease of reference, we reproduce the empirical comparison against D-MPC on D4RL datasets from Section~\ref{sec:exp} in Table~\ref{tab:dmpc_locomotion_app}.

\begin{table}[h]
\centering
\small
\renewcommand{\arraystretch}{1.4}
\begin{tabular}{@{}l >{\raggedright\arraybackslash}p{4.5cm} >{\raggedright\arraybackslash}p{5.0cm}@{}}
\toprule
 & \textbf{Zhou et al.} & \textbf{Ours} \\
\midrule
What is optimized?      & Action sequence (best candidate) & Policy parameters \\
How is it optimized?    & Gradient-free (random shooting)  & Gradient (BPTT) \\
Dynamics                & Multi-step trajectory diffusion  & One-step transition diffusion \\
Action source           & Learned diffusion proposal       & Pretrained behavior policy \\
\bottomrule
\end{tabular}
\vspace{2ex}
\caption{Design-space comparison between D-MPC and our approach.}
\label{tab:dmpc-comparison}
\end{table}

\begin{table}[!htbp]
\centering
\renewcommand{\arraystretch}{1.05}
\begin{tabular}{l|cc}
\toprule
Task Name & D-MPC & $\DiffMPC$ \\
\midrule
halfcheetah medium        & $46.00 \pm 0.17$ & $\textbf{70.05} \pm 1.8$ \\
hopper medium             & $61.24 \pm 2.30$ & $\textbf{103.38} \pm 0.35$ \\
walker2d medium           & $76.21 \pm 2.67$ & $\textbf{88.91} \pm 0.6$ \\
halfcheetah medium-replay & $41.12 \pm 0.31$ & $\textbf{59.89} \pm 1.2$ \\
hopper medium-replay      & $92.49 \pm 2.23$ & $\textbf{103.14} \pm 0.45$ \\
walker2d medium-replay    & $78.81 \pm 4.19$ & $\textbf{95.87} \pm 1.19$ \\
\midrule
Average & $65.98$ & $\textbf{86.87}$ \\
\bottomrule
\end{tabular}
\caption{Average normalized score on ten unseen seeds. $\pm$ represents the standard deviation across seeds.}
\label{tab:dmpc_locomotion_app}
\end{table}

\section{FlowMPC: One step Model}
\label{sec:flowmpc_app}
Inference-time policy optimization can improve performance over purely offline-trained policies, but it brings this issue of deployment speed to the forefront. The inner loop requires many model evaluations per environment step, and diffusion-based transition models are especially costly because each simulated transition requires iterative denoising \citep{ho2020denoising}. As a result, iterative sampling becomes a dominant latency bottleneck in real-time control settings. There is a strong need for faster generative dynamics for differentiable planning.

One way to reduce this deployment cost is to replace iterative diffusion sampling with a \emph{single-step} generative model. Mean-flow \citep{geng2025meanflows} dynamics provide such a mechanism: the model maps a state-action pair to the next-state prediction in one forward pass, cutting per-transition latency at inference. This is especially attractive in inference-time MPC loops, where each environment step may require many simulated transitions and wall-clock time scales directly with the number of denoising steps. The main drawback is statistical: single-step models are typically less expressive than multi-step diffusion samplers and can incur higher prediction error on complex, multimodal transition distributions, and under rollout these errors compound. To  mitigate this, we  bias world-model training toward state-action regions that are likely under a strong pretrained policy, so that capacity is spent where the deployed controller will actually query the model. 

\subsection{Mean-Flow Models}
Mean-flow models \citep{geng2025meanflows} provide a framework for \emph{one-step} generation by modeling an \emph{average} velocity field, rather than the \emph{instantaneous} velocity used in standard Flow Matching \citep{lipman2022flowmatching}.
We follow the common continuous-time transport view, where a latent variable $z_t\in\mathbb{R}^d$ evolves over $t\in[0,1]$ to map between a data distribution and a simple prior.
Let $x\sim p_{\mathrm{data}}$ denote a data sample and $\epsilon\sim p_{\mathrm{prior}}$ denote a prior sample (typically $\mathcal{N}(0,I)$).
A simple probability path used in rectified-flow style constructions \citep{liu2022rectifiedflow} is the linear interpolation
\begin{align*}
z_t \;:=\; (1-t)\,x + t\,\epsilon.
\end{align*}
Along such paths, Flow Matching learns an instantaneous velocity field $v(z_t,t)$ that defines an ODE
\begin{align*}
\frac{d}{dt} z_t \;=\; v(z_t,t),
\end{align*}
so that generation can be performed by numerically solving the ODE from $t=1$ to $t=0$ starting at $z_1=\epsilon$.

\paragraph{Average velocity.}
Mean-flow models instead define the \emph{average velocity} between two time points $r<t$ as \citep{geng2025meanflows}
\begin{align*}
u(z_t,r,t) \;:=\; \frac{1}{t-r}\int_{r}^{t} v(z_{\tau},\tau)\,d\tau.
\end{align*}
This quantity is aligned with the displacement along the trajectory: $z_r \;=\; z_t - (t-r)\,u(z_t,r,t)$.

If a neural network $u_{\theta}$ accurately predicts $u(\cdot,r,t)$, then sampling can avoid numerical integration.
In particular, \emph{one-step} sampling is obtained by drawing $z_1=\epsilon\sim p_{\mathrm{prior}}$ and setting
\begin{align*}
z_0 \;=\; z_1 - u_{\theta}(z_1,0,1).
\end{align*}

\paragraph{MeanFlow Identity and Training Objective.}
Direct supervision of $u(z_t,r,t)$ is intractable because it is defined by a time integral.
\citet{geng2025meanflows} derive an identity relating average and instantaneous velocities:
\begin{align*}
u(z_t,r,t)
\;=\;
v(z_t,t) - (t-r)\,\frac{d}{dt}u(z_t,r,t),
\end{align*}
where $\frac{d}{dt}$ is the total derivative along the trajectory.
Expanding the total derivative gives
\begin{align*}
\frac{d}{dt}u(z_t,r,t)
\;=\;
\big\langle \nabla_{z}u(z_t,r,t),\, v(z_t,t)\big\rangle + \partial_t u(z_t,r,t).
\end{align*}
This yields a tractable training target by replacing $v(z_t,t)$ with the usual conditional velocity used in Flow Matching.
For the linear path above, a standard conditional velocity is $v_{\mathrm{cond}}=\epsilon-x$ \citep{lipman2022flowmatching,liu2022rectifiedflow}.
MeanFlow training fits $u_{\theta}$ by regressing to the identity-based target:
\begin{align*}
\mathcal{L}_{\mathrm{MeanFlow}}(\theta)
=
\mathbb{E}\Big[
\big\|
u_{\theta}(z_t,r,t)
-
\mathrm{stopgrad}\big(
v_{\mathrm{cond}} - (t-r)\,\tfrac{d}{dt}u_{\theta}(z_t,r,t)
\big)
\big\|_2^2
\Big],
\end{align*}
with expectations over $x\sim p_{\mathrm{data}}$, $\epsilon\sim p_{\mathrm{prior}}$, and a sampling scheme for $(r,t)$ with $0\le r<t\le 1$.
In our world-model setting, $x$ corresponds to the next state (or state increment), and both $u_{\theta}$ and the conditional velocity are parameterized \emph{conditionally} on the current state--action pair.

\subsection{One-Step Mean-Flow Dynamics for World Modeling}
We learn a conditional, differentiable transition sampler for next-state generation. Given $(s_t,a_t)$, we generate $s_{t+1}$ via a reparameterized map
\begin{align*}
s_{t+1} \;=\; f_{\theta}(s_t,a_t,\epsilon_t),
\qquad
\epsilon_t \sim p_0(\epsilon),
\end{align*}
so that rollout sampling is a single forward pass and gradients with respect to $(s_t,a_t)$ are available.

\paragraph{MeanFlow sampler.}
We instantiate $f_{\theta}$ using a one-step MeanFlow model \citep{geng2025meanflows}. The model learns an \emph{average-velocity} field $u_{\theta}(\cdot,r,t\,;\,s,a)$ and uses the MeanFlow sampling $z_r = z_t - (t-r)\,u(z_t,r,t)$. In the one-step case with $(r,t)=(0,1)$ and $z_1=\epsilon$, generation is $z_0=\epsilon-u_{\theta}(\epsilon,0,1)$. We therefore define the conditional transition sampler as
\begin{align*}
f_{\theta}(s_t,a_t,\epsilon_t)
\;:=\;
\epsilon_t \;-\; u_{\theta}(\epsilon_t,0,1\,;\,s_t,a_t).
\end{align*}

\paragraph{Training from offline transitions.}
We train on $\mathcal{D}=\{(s_t,a_t,r_t,s_{t+1})\}$ by treating $x=s_{t+1}$ as the data sample. We draw $\epsilon\sim p_0$ and times $0\le r<t\le 1$, and form the linear interpolation $z=(1-t)x+t\epsilon$. For this path, the conditional velocity is $v=\epsilon-x$ \citep{lipman2022flowmatching}. MeanFlow fits $u_{\theta}$ by enforcing the MeanFlow identity \citep{geng2025meanflows} through the regression target
\[
u_{\mathrm{tgt}} \;=\; v \;-\; (t-r)\,\frac{d}{dt}u_{\theta}(z,r,t\,;\,s_t,a_t),
\]
where the total derivative along the path is $\frac{d}{dt}u_{\theta} = v\,\partial_z u_{\theta} + \partial_t u_{\theta}$ (implemented as a Jacobian-vector product). The training loss is
\begin{align*}
\mathcal{L}_{\mathrm{MF}}(\theta)
\;=\;
\mathbb{E}\Big[
\big\|
u_{\theta}(z,r,t\,;\,s_t,a_t)
-
\mathrm{sg}\!\big(u_{\mathrm{tgt}}\big)
\big\|_2^2
\Big],
\end{align*}
where $\mathrm{sg}(\cdot)$ denotes stop-gradient.

\subsection{Policy-Tilted Flow World Models}

A one-step world model trades compute for statistical error: it is fast to sample, but rollout quality depends on where the model is accurate. Since inference-time planning queries transitions induced by the deployed controller, we train the world model under a \emph{tilted} distribution that emphasizes state--action regions likely under a strong pretrained policy $\pi_{\psi_{\mathrm{pre}}}$.

\paragraph{Tilted training objective.}
Let $\ell_{\theta}(s,a,s')$ denote the per-transition MeanFlow loss from the previous subsection (the contribution inside $\mathcal{L}_{\mathrm{MF}}(\theta)$) for a transition $(s,a,s')\in\mathcal{D}$. We train the transition model by minimizing the weighted risk
\begin{align*}
\mathcal{L}_{\mathrm{tilt}}(\theta)
\;=\;
\mathbb{E}_{(s,a,s')\sim \mathcal{D}}
\Big[
\bar{w}_{\eta}(s,a)\,\ell_{\theta}(s,a,s')
\Big],
\end{align*}
where the normalized tilt weight is
\begin{align*}
\bar{w}_{\eta}(s,a)
\;:=\;
\frac{\exp\!\big(\eta \,\log \pi_{\psi_{\mathrm{pre}}}(a\mid s)\big)}
{\mathbb{E}_{(s,a)\sim \mathcal{D}}\!\left[\exp\!\big(\eta \,\log \pi_{\psi_{\mathrm{pre}}}(a\mid s)\big)\right]},
\end{align*}
and $\eta\ge 0$ controls the tilt strength (with $\eta=0$ recovering uniform training). In practice, the denominator is estimated by the minibatch mean so that the average weight in each minibatch is $1$.

The weight $\bar{w}_{\eta}(s,a)$ increases for dataset transitions whose actions have higher likelihood under $\pi_{\psi_{\mathrm{pre}}}(\cdot\mid s)$. This concentrates model capacity on state--action regions that dominate imagined rollouts at deployment, improving rollout fidelity in the MPC loop and mitigating the statistical error introduced by using a one-step MeanFlow transition model.

\subsection{Experiments}

\begin{table*}[t]
\centering
\small
\begin{tabular}{l|cc}
\toprule
Task Name & Zero Shot & $\FlowMPC$ \\
\midrule
halfcheetah-medium        & $65.42 \pm 1.76$  & $\textbf{70.06} \pm 1.9$ \\
halfcheetah-expert        & $105.33 \pm 3.21$ & $\textbf{106.96} \pm 1.486$ \\
halfcheetah-medium-expert & $103.79 \pm 2.8$  & $\textbf{109.89} \pm 0.89$ \\
halfcheetah-medium-replay & $50.78 \pm 0.56$  & $\textbf{56.51} \pm 5.5$ \\
halfcheetah-full-replay   & $82.38 \pm 1.32$  & $\textbf{85.89} \pm 0.95$ \\
\midrule
hopper-medium             & $102.25 \pm 0.23$ & $\textbf{103.07} \pm 0.65$ \\
hopper-expert             & $100.74 \pm 14.9$ & $\textbf{101.21} \pm 17.2$ \\
hopper-medium-expert      & $107.48 \pm 4.5$  & $\textbf{107.78} \pm 3.06$ \\
hopper-medium-replay      & $98.65 \pm 5.15$  & $\textbf{102.60} \pm 0.59$ \\
hopper-full-replay        & $107.49 \pm 0.74$ & $\textbf{108.36} \pm 0.55$ \\
\midrule
walker2d-medium           & $84.79 \pm 1.52$  & $\textbf{88.59} \pm 1.2$ \\
walker2d-expert           & $112.19 \pm 0.31$ & $\textbf{116.44} \pm 0.42$ \\
walker2d-medium-expert    & $111.65 \pm 0.46$ & $\textbf{115.01} \pm 1.09$ \\
walker2d-medium-replay    & $83.86 \pm 7.38$  & $\textbf{90.92} \pm 4.2$ \\
walker2d-full-replay      & $101.11 \pm 3.9$  & $\textbf{104.76} \pm 1.12$ \\
\midrule
Average & $94.93$ & $\textbf{97.87}$ \\
\bottomrule
\end{tabular}
\caption{Average normalized score over the final evaluation and ten unseen training seeds on Gym-MuJoCo tasks, excluding random datasets. The symbol $\pm$ represents the standard deviation across the seeds.}
\label{tab:rebrac_mujoco_meanflow}
\end{table*}

\paragraph{Benchmarks and evaluation.}
We evaluate on the D4RL Gym-MuJoCo locomotion suite \citep{fu2020d4rl} across HalfCheetah, Hopper, and Walker2d, using the standard v2 datasets in five regimes (medium, expert, medium-expert, medium-replay, full-replay), for a total of $15$ tasks. We report D4RL normalized scores computed from episode returns. Each reported value is the mean and standard deviation over $10$ evaluation seeds. 

\paragraph{Methods compared.}
ReBRAC is the offline-trained initialization policy \citep{tarasov2023minimalist} and is the zero shot performance. {$\FlowMPC$} is our method, which replaces diffusion transition sampling with a one-step MeanFlow transition model and trains the world model with policy-tilted weighting.

\paragraph{Inference-time policy improvement.}
Both $\DiffMPC$ and $\FlowMPC$ refine the deployed controller at test time. At each environment step, they generate imagined forward rollouts under a learned differentiable transition model, score trajectories using learned rewards and a terminal critic term, and take a small number of gradient steps through the rollout computation graph before executing the resulting action. The key difference is the transition sampler: $\DiffMPC$ uses a conditional diffusion model with $K$ reverse denoising steps per transition, while $\FlowMPC$ uses a single-step MeanFlow sampler.

\paragraph{Main results on D4RL.}
Table~\ref{tab:rebrac_mujoco_meanflow} summarizes performance on all tasks. $\FlowMPC$ improves over the ReBRAC initialization on average (83.91 vs.\ 81.14) and matches or exceeds MPCwDWM on several HalfCheetah regimes. MPCwDWM attains the best average score (85.18), but it relies on iterative diffusion sampling at inference time which is far more expensive.

\begin{table}[t]
\centering
\small
\begin{tabular}{c|c|c|c}
\toprule
Diffusion Steps & Diffusion (ms) & MeanFlow (ms) & Speedup \\
\midrule
\multicolumn{4}{c}{\textbf{HalfCheetah}} \\
\midrule
10 & $1.0200 \pm 0.0019$ & $0.1177 \pm 0.0017$ & $8.67$x \\
20 & $2.0184 \pm 0.0019$ & $0.1177 \pm 0.0017$ & $17.15$x \\
30 & $3.0852 \pm 0.0008$ & $0.1177 \pm 0.0017$ & $26.21$x \\
40 & $4.0280 \pm 0.0011$ & $0.1177 \pm 0.0017$ & $34.22$x \\
50 & $5.0487 \pm 0.0072$ & $0.1177 \pm 0.0017$ & $42.89$x \\
\midrule
\multicolumn{4}{c}{\textbf{Hopper}} \\
\midrule
10 & $1.0211 \pm 0.0017$ & $0.1213 \pm 0.0018$ & $8.42$x \\
20 & $2.0214 \pm 0.0017$ & $0.1213 \pm 0.0018$ & $16.66$x \\
30 & $3.0788 \pm 0.0010$ & $0.1213 \pm 0.0018$ & $25.38$x \\
40 & $4.0237 \pm 0.0038$ & $0.1213 \pm 0.0018$ & $33.17$x \\
50 & $5.0508 \pm 0.0017$ & $0.1213 \pm 0.0018$ & $41.64$x \\
\midrule
\multicolumn{4}{c}{\textbf{Walker2d}} \\
\midrule
10 & $1.0252 \pm 0.0017$ & $0.1182 \pm 0.0019$ & $8.67$x \\
20 & $2.0225 \pm 0.0018$ & $0.1182 \pm 0.0019$ & $17.11$x \\
30 & $3.0796 \pm 0.0016$ & $0.1182 \pm 0.0019$ & $26.05$x \\
40 & $4.0387 \pm 0.0027$ & $0.1182 \pm 0.0019$ & $34.17$x \\
50 & $5.0597 \pm 0.0007$ & $0.1182 \pm 0.0019$ & $42.81$x \\
\bottomrule
\end{tabular}
\caption{Forward-pass latency comparison between diffusion and MeanFlow on D4RL random datasets (5 runs per measurement). Horizon $H=1$, batch size $N=2048$. Action dimension: HalfCheetah $=6$, Hopper $=3$, Walker2d $=6$. MeanFlow uses a single step while diffusion uses the number of denoising steps shown. Values report mean $\pm$ standard deviation in milliseconds.}
\label{tab:latency_forward}
\end{table}

\begin{table}[t]
\centering
\small
\begin{tabular}{c|c|c|c}
\toprule
Diffusion Steps & Diffusion Total (ms) & MeanFlow Total (ms) & Speedup \\
\midrule
\multicolumn{4}{c}{\textbf{HalfCheetah}} \\
\midrule
10 & $3.2964 \pm 0.0011$ & $0.2184 \pm 0.0010$ & $15.09$x \\
20 & $6.5419 \pm 0.0019$ & $0.2184 \pm 0.0010$ & $29.95$x \\
30 & $9.8108 \pm 0.0016$ & $0.2184 \pm 0.0010$ & $44.92$x \\
40 & $13.0628 \pm 0.0009$ & $0.2184 \pm 0.0010$ & $59.81$x \\
50 & $16.3602 \pm 0.0011$ & $0.2184 \pm 0.0010$ & $74.91$x \\
\midrule
\multicolumn{4}{c}{\textbf{Hopper}} \\
\midrule
10 & $3.3223 \pm 0.0012$ & $0.2261 \pm 0.0012$ & $14.69$x \\
20 & $6.5865 \pm 0.0003$ & $0.2261 \pm 0.0012$ & $29.13$x \\
30 & $9.8816 \pm 0.0019$ & $0.2261 \pm 0.0012$ & $43.70$x \\
40 & $13.1673 \pm 0.0006$ & $0.2261 \pm 0.0012$ & $58.24$x \\
50 & $16.4745 \pm 0.0021$ & $0.2261 \pm 0.0012$ & $72.86$x \\
\midrule
\multicolumn{4}{c}{\textbf{Walker2d}} \\
\midrule
10 & $3.2956 \pm 0.0014$ & $0.2182 \pm 0.0008$ & $15.10$x \\
20 & $6.5398 \pm 0.0003$ & $0.2182 \pm 0.0008$ & $29.97$x \\
30 & $9.8225 \pm 0.0022$ & $0.2182 \pm 0.0008$ & $45.02$x \\
40 & $13.0738 \pm 0.0014$ & $0.2182 \pm 0.0008$ & $59.92$x \\
50 & $16.3534 \pm 0.0007$ & $0.2182 \pm 0.0008$ & $74.95$x \\
\bottomrule
\end{tabular}
\caption{Total gradient backpropagation time comparison between diffusion and MeanFlow on D4RL random datasets (5 runs). Horizon $H=1$, batch size $N=2048$. Action dimension: HalfCheetah $=6$, Hopper $=3$, Walker2d $=6$. MeanFlow uses one step while diffusion uses the number of denoising steps shown. Values report mean $\pm$ standard deviation in milliseconds.}
\label{tab:latency_total}
\end{table}

\paragraph{Inference-time compute and latency.}
Tables~\ref{tab:latency_forward} and \ref{tab:latency_total} measure the cost of a single transition evaluation and the corresponding forward+backward cost used in inference-time optimization.
Across environments, one-step MeanFlow reduces the forward-pass cost by about $8\times$ relative to a $K=10$ diffusion sampler, with the gap widening linearly as the number of diffusion denoising steps increases.
The same trend holds for end-to-end differentiation: MeanFlow reduces total (forward+backward) time by roughly $15\times$ at $K=10$ and by more than $70\times$ at $K=50$.
These measurements quantify the bottleneck addressed by $\FlowMPC$: rollout generation and differentiation dominate deployment-time compute, and replacing iterative denoising with a one-step sampler yields large, predictable latency reductions.

\clearpage
\section{Broader impacts}
\label{sec:impact}
This work is foundational methods research on offline reinforcement learning, evaluated on standard simulated continuous-control benchmarks (D4RL MuJoCo and AntMaze) rather than on a deployed system. On the positive side, offline RL with inference-time adaptation is motivated by settings where online interaction is costly or unsafe (e.g., healthcare, robotics, autonomous driving, recommender systems), and improving the reliability of policies extracted from fixed datasets can reduce the need for risky exploration in such domains. On the negative side, the same techniques could in principle be applied to high-stakes decision-making systems where errors have real consequences; inference-time policy updates also make the deployed policy state-dependent and harder to audit than a fixed pretrained policy, which could complicate verification in safety-critical applications. We do not release any new datasets or pretrained models that pose elevated misuse risk beyond what is already available through D4RL and standard offline RL baselines, and we view careful domain-specific validation, uncertainty quantification, and human oversight as appropriate mitigations before any real-world deployment.

\end{document}